\documentclass[twoside,11pt]{article}
\usepackage{jmlr2f}

\usepackage[l2tabu, orthodox]{nag}
\usepackage{mathtools}
\mathtoolsset{
showonlyrefs=true 
}
\usepackage[ruled,vlined]{algorithm2e}
\usepackage{amsmath, subfigure, dsfont, amsfonts, amssymb, graphicx, enumerate, xspace, cancel, verbatim, comment, cite}
\usepackage[colorlinks,pdfpagelabels,citecolor=blue,plainpages=false]{hyperref}

\newcommand{\eod}{{${}$\\}}

        {\hspace*{\fill}$\Box$\par\vspace{4mm}}

\newcommand*\bell{\ensuremath{\boldsymbol\ell}}

\newcommand*\bepsilon{\ensuremath{\boldsymbol\epsilon}}
\newcommand*\bpi{\ensuremath{\boldsymbol\pi}}
\newcommand*\bphi{\ensuremath{\boldsymbol\phi}}
\newcommand*\bpsi{\ensuremath{\boldsymbol\psi}}
\newcommand*\bmu{\ensuremath{\boldsymbol\mu}}
\newcommand*\bZe{\ensuremath{\boldsymbol 0}}

\DeclareMathOperator*{\argmin}{argmin}
\DeclareMathOperator*{\argmax}{argmax}
\DeclareMathOperator*{\grad}{\nabla}
\DeclareMathOperator*{\expec}{\mathbb E}

\newcommand{\error}{{\mathcal{E}}}

\newcommand{\x}{{\mathbf x}}

\newcommand{\subg}{{\mathbf g}}

\newcommand{\y}{{\mathbf y}}

\newcommand{\z}{{\mathbf z}}

\newcommand{\vt}{\mathbf{v}}
\newcommand{\bu}{\mathbf{u}}
\newcommand{\wt}{\mathbf{w}}
\newcommand{\Wt}{\mathbf{W}}
\newcommand{\Vt}{\mathbf{V}}

\newcommand{\Qt}{\mathbf{G}}

\newcommand{\bI}{\mathbf{I}}
\newcommand{\tdg}{\xi}

\newcommand{\sta}{\mathbf{s}}
\newcommand{\act}{\mathbf{a}}

\newcommand{\bP}{{\mathbb P}}
\newcommand{\bR}{{\mathbb R}}
\newcommand{\bbO}{{\mathds 1}}

\newcommand{\cA}{{\mathcal A}}

\newcommand{\cS}{{\mathcal S}}

\newcommand{\ccH}{{\mathcal H}}
\newcommand{\qprop}{{\texttt{GProp}}}
\newcommand{\reinforce}{{\texttt{REINFORCE}}}

\newcommand{\copdac}{{\texttt{COPDAC-}\mathtt{Q}}}

\newcommand{\dd}{{\partial}}

\newtheorem{thm}{Theorem}

\newtheorem{lem}[thm]{Lemma}
\newtheorem{defn}{Definition}

\newtheorem{eg}{Example}
\newtheorem{rem}{Remark}

\usepackage{listings}
\usepackage{caption}

\DeclareCaptionFormat{listing}{\rule{\dimexpr\textwidth+17pt\relax}{0.4pt}\par\vskip1pt#1#2#3}
\DeclareMathVersion{sans}
\SetSymbolFont{operators}{sans}{OT1}{cmbr}{m}{n}
\SetSymbolFont{letters}  {sans}{OML}{cmbrm}{m}{it}
\SetSymbolFont{symbols}  {sans}{OMS}{cmbrs}{m}{n}

\lstnewenvironment{sflisting}[1][]
  {\lstset{#1}\mathversion{sans}}{}

\jmlrheading{1}{2000}{1-48}{4/00}{10/00}{David Balduzzi and Muhammad Ghifary}

\ShortHeadings{Compatible Value Gradients for Deep Reinforcement Learning}{Balduzzi and Ghifary}
\firstpageno{1}
\begin{document} 
\title{Compatible Value Gradients for Reinforcement Learning\\ of Continuous Deep Policies}
\author{\name David Balduzzi \email david.balduzzi@vuw.ac.nz \\
       \addr School of Mathematics and Statistics\\
       Victoria University of Wellington\\
       Wellington, New Zealand
       \AND
       \name Muhammad Ghifary \email muhammad.ghifary@ecs.vuw.ac.nz \\
       \addr School of Engineering and Computer Science\\
       Victoria University of Wellington\\
       Wellington, New Zealand
}

\maketitle

\begin{abstract} 
	This paper proposes $\qprop$, a deep reinforcement learning algorithm for continuous policies with compatible function approximation. The algorithm is based on two innovations. Firstly, we present a temporal-difference based method for learning the \emph{gradient} of the value-function. Secondly, we present the deviator-actor-critic (DAC) model, which comprises three neural networks that estimate the value function, its gradient, and determine the actor's policy respectively.

	We evaluate $\qprop$ on two challenging tasks: a contextual bandit problem constructed from nonparametric regression datasets that is designed to probe the ability of reinforcement learning algorithms to accurately estimate gradients; and the octopus arm, a challenging reinforcement learning benchmark. $\qprop$ is competitive with \emph{fully supervised methods} on the bandit task and achieves the best performance to date on the octopus arm.
\end{abstract}


\section{Introduction}
\label{sec:intro}

In reinforcement learning, an agent learns to maximize its discounted future rewards \citep{sutton:98}. The structure of the environment is initially unknown, so the agent must both learn the rewards associated with various action-sequence pairs and optimize its policy. A natural approach is to tackle the subproblems separately via a critic and an actor \citep{barto:83,konda:00}, where the critic estimates the value of different actions and the actor maximizes rewards by following the policy gradient \citep{sutton:99,peters:06,silver:14}. Policy gradient methods have proven useful in settings with high-dimensional continuous action spaces, especially when task-relevant \emph{policy representations} are at hand \citep{deisenroth:11,levine:15,wahlstrom:15}.

We tackle the problem of learning actor (policy) and critic representations. In the supervised setting, representation or deep learning algorithms have recently demonstrated remarkable performance on a range of benchmark problems. However, the problem of learning features for reinforcement learning remains comparatively underdeveloped. The most dramatic recent success uses $Q$-learning over finite action spaces, and essentially build a neural network critic \citep{Mnih:2015wq}. Here, we consider \emph{continuous} action spaces, and develop an algorithm that simultaneously learns the value function and its gradient, which it then uses to find the optimal policy.

\subsection{Outline}
This paper presents Value-Gradient Backpropagation ($\qprop$), a deep actor-critic algorithm for continuous action spaces with compatible function approximation. Our starting point is the deterministic policy gradient and associated compatibility conditions derived in \citep{silver:14}. Roughly speaking, the compatibility conditions are that
\begin{enumerate}[C1.]
	\item the critic approximate the gradient of the value-function and
	\item the approximation is closely related to the gradient of the policy.
\end{enumerate}
See Theorem~\ref{t:compat} for details. We identify and solve two problems with prior work on policy gradients -- relating to the two compatibility conditions:
\begin{enumerate}[P1.]
 	\item \emph{Temporal difference methods do not directly estimate the gradient of the value function.}\\
 	Instead, temporal difference methods are applied to learn an approximation of the form $Q^\vt(\sta) + Q^\wt(\sta,\act)$, where $Q^\vt(\sta)$ estimates the value of a state, given the current policy, and $Q^\wt(\sta,\act)$ estimates the \emph{advantage} from deviating from the current policy \citep{sutton:99,peters:06,deisenroth:11,silver:14}. Although the advantage is related to the gradient of the value function, it is not the same thing.
 	\item \emph{The representations used for compatible approximation scale badly on neural networks.}\\
 	The second problem is that prior work has restricted to advantage functions constructed from a particular state-action representation, $\bphi(\sta,\act) = \grad_\theta\bmu_\theta(\sta)(\act-\bmu_\theta(\sta))$, that depends on the gradient of the policy. The representation is easy to handle for linear policies. However, if the policy is a neural network, then the standard state-action representation ties the critic too closely to the actor and depends on the internal structure of the actor, Example~\ref{eg:deep_advantage}. As a result, weight updates cannot be performed by backpropagation, see section~\ref{sec:problem}. 
 \end{enumerate} 

\noindent
The paper makes three novel contributions. The first two contributions relate directly to problems P1 and P2. The third is a new task designed to test the accuracy of gradient estimates.

\paragraph{Method to directly learn the gradient of the value function.}
The first contribution is to modify temporal difference learning so that it directly estimates the gradient of the value-function. The \emph{gradient perturbation trick}, Lemma~\ref{lem:gradient}, provides a way to simultaneously estimate both the value of a function at a point and its gradient, by perturbing the function's input with uncorrelated Gaussian noise. 

Plugging in a neural network instead of a linear estimator extends the trick to the problem of learning a function and its gradient over the entire state-action space. Moreover, the trick combines naturally with temporal difference methods, Theorem~\ref{thm:extension}, and is therefore well-suited to applications in reinforcement learning.

\paragraph{Deviator-Actor-Critic (DAC) model with compatible function approximation.}
The second contribution is to propose the Deviator-Actor-Critic (DAC) model, Definition~\ref{def:beh_crit}, consisting in three coupled neural networks and Value-Gradient Backpropagation ($\qprop$), Algorithm~\ref{alg:qprop}, which backpropagates three different signals to train the three networks. The main result, Theorem~\ref{thm:main}, is that $\qprop$ has compatible function approximation when implemented on the DAC model when the neural network consists in linear and rectilinear units.\footnote{The proof also holds for maxpooling, weight-tying and other features of convnets. A description of how closely related results extend to convnets is provided in \citep{doco:15}.} 

The proof relies on decomposing the Actor-network into individual units that are considered as actors in their own right, based on ideas in \citep{srivastava:14,doco:15}. It also suggests interesting connections to work on structural credit assignment in multiagent reinforcement learning \citep{agogino:04,agogino:08,holmesparker:14}.

\paragraph{Contextual bandit task to probe the accuracy of gradient estimates.}
A third contribution, that may be of independent interest, is a new contextual bandit setting designed to probe the ability of reinforcement learning algorithms to estimate gradients. A supervised-to-contextual bandit transform was proposed in \citep{dudik:14} as a method for turning classification datasets into $K$-armed contextual bandit datasets.

We are interested in the \emph{continuous} setting in this paper. We therefore adapt their transform with a twist. The SARCOS and Barrett datasets from robotics have features corresponding to the positions, velocities and accelerations of seven joints and labels corresponding to their torques. There are 7 joints in both cases, so the feature and label spaces are 21 and 7 dimensional respectively. The datasets are traditionally used as regression benchmarks labeled SARCOS1 through SARCOS7 where the task is to predict the torque of a single joint -- and similarly for Barrett.

We convert the two datasets into two continuous contextual bandit tasks where the reward signal is the negative distance to the correct label 7-dimensional. The algorithm is thus ``told'' that the label lies on a sphere in a 7-dimensional space. The missing information required to pin down the label's position is precisely the gradient. For an algorithm to make predictions that are competitive with fully supervised methods, it is necessary to find extremely accurate gradient estimates.

\paragraph{Experiments.}
Section~\ref{sec:experiments} evaluates the performance of $\qprop$ on the contextual bandit problems described above and on the challenging octopus arm task \citep{engel:05}. We show that $\qprop$ is able to simultaneously solve seven nonparametric regression problems without observing any labels -- instead using the distance between its actions and the correct labels. It turns out that $\qprop$ is competitive with recent \emph{fully supervised} learning algorithms on the task. Finally, we evaluate $\qprop$ on the octopus arm benchmark, where it achieves the best performance reported to date.

\subsection{Related work}

An early reinforcement learning algorithm for neural networks is $\reinforce$ \citep{williams:92}. A disadvantage of $\reinforce$ is that the entire network is trained with a single scalar signal. 

Our proposal builds on ideas introduced with deep $Q$-learning \citep{Mnih:2015wq}, such as replay. However, deep $Q$-learning is restricted to finite action spaces, whereas we are concerned with \emph{continuous} action spaces. 

Policy gradients were introduced in \citep{sutton:99} and have been used extensively \citep{kakade:01,peters:06,deisenroth:11}. The deterministic policy gradient was introduced in \citep{silver:14}, which also proposed the algorithm $\copdac$. The relationship between $\qprop$ and $\copdac$ is discussed in detail in section~\ref{sec:problem}.

An alternate approach, based on the idea of backpropagating the gradient of the value function, is developed in \citep{jordan:90,prokhorov:97,wang:01,hafner:11,fairbank:12,fairbank:13}. Unfortunately, these algorithms do not have compatible function approximation in general, so there are no guarantees on actor-critic interactions. See section~\ref{sec:problem} for further discussion.

The analysis used to prove compatible function approximation relies on decomposing the Actor neural network into a collection of agents corresponding to the units in the network. The relation between $\qprop$ and the difference-based objective proposed for multiagent learning \citep{agogino:08,holmesparker:14} is discussed in section~\ref{sec:local_actors}.

\subsection{Notation}

We use boldface to denote vectors, subscripts for time, and superscripts for individual units in a network. Sets of parameters are capitalized ($\Theta$, $\Wt$, $\Vt$) when they refer to matrices or to the parameters of neural networks.

\section{Deterministic Policy Gradients}
\label{sec:ac}

This section recalls previous work on policy gradients. The basic idea is to simultaneously train an actor and a critic. The critic learns an estimate of the value of different policies; the actor then follows the gradient of the value-function to find an optimal (or locally optimal) policy in terms of expected rewards.

\subsection{The Policy Gradient Theorem}

The environment is modeled as a Markov Decision Process consisting of state space $\cS\subset\bR^m$, action space $\cA\subset\bR^d$, initial distribution $p_1(\sta)$ on states, stationary transition distribution $p(\sta_{t+1}|\sta_t,\act_t)$ and reward function $r:\cS\times \cA\rightarrow \bR$. A \emph{policy} is a function $\bmu_\theta:\cS\rightarrow \cA$ from states to actions. We will often add noise to policies, causing them to be stochastic. In this case, the policy is a function $\bmu_\theta:\cS\rightarrow \triangle_\cA$, where $\triangle_\cA$ is the set of probability distributions on actions.

Let $p_t(\sta\rightarrow \sta',\bmu)$ denote the distribution on states $\sta'$ at time $t$ given policy $\bmu$ and initial state $\sta$ at $t=0$ and let $\rho^{\bmu}(\sta') = \int_\cS\sum_{t=0}^\infty\gamma^tp_1(\sta)p_t(\sta\rightarrow \sta',\bmu)d\sta$. Let $r_t^\gamma = \sum_{\tau=t}^\infty\gamma^{\tau-t} r(\sta_\tau,\act_\tau)$ be the discounted future reward. Define the
\begin{align}
	 \text{value of a state-action pair: } \qquad
	 & Q^{\bmu_\theta}(\sta,\act) = \expec[r_1^\gamma | {\mathbf S}_1=\sta,{\mathbf A}_1=\act;\bmu_\theta]\quad\text{and}
	\label{e:Q}
	\\
	\text{value of a policy: } \qquad
	&  J({\bmu_\theta}) = \expec_{\sta\sim \rho^{\bmu}, \act\sim \bmu_\theta}[Q^{\bmu_\theta}(\sta,\act)].
	\label{eq:J}
\end{align}
The aim is to find the policy $\theta^* := \argmax_\theta J({\bmu_\theta})$ with maximal value.  A natural approach is to follow the gradient \citep{sutton:99}, which in the deterministic case can be computed explicitly as
\begin{thm}[policy gradient]\label{t:dpg}\eod	
	Under reasonable assumptions on the regularity of the Markov Decision Process the policy gradient can be computed as
		\begin{align}
			\grad_\theta J(\bmu_\theta)
			& = \expec_{\sta\sim \rho^{\bmu}} \left[\grad_\theta\bmu_\theta(\sta) \grad_{\act} Q^{\bmu}(\sta,\act)_{|\act=\bmu_\theta(\sta)}\right].
		\end{align}	
\end{thm}

\begin{proof}
	See \citep{silver:14}.
\end{proof}

\subsection{Linear Compatible Function Approximation}

Since the agent does not have direct access to the value function $Q^{\bmu}$, it must instead learn an estimate $Q^\wt\approx Q^{\bmu}$. A sufficient condition for when plugging an estimate $Q^\wt(\sta,\act)$ into the policy gradient $\grad_\theta J(\theta) = \expec[\grad_\theta \bmu_\theta(\sta)\grad_{\act} Q^{\bmu_\theta}(\sta,\act)_{|\act=\bmu_\theta(\sta)}]$ yields an unbiased estimator was first proposed in \citep{sutton:99}. 

A sufficient condition in the deterministic setting is:

\begin{thm}[compatible value function approximation]\label{t:compat}\eod
 	The value-estimate $Q^{\wt}(\sta,\act)$ satisfies is compatible with the policy gradient, that is 
 	\begin{equation}
 		\grad_\theta J(\theta) = \expec\left[\grad_\theta \bmu_\theta(\sta)\cdot \grad_{\act} Q^{\wt}(\sta,\act)_{|\act=\bmu_\theta(\sta)}\right]
 	\end{equation}
 	if the following conditions hold:
		\begin{enumerate}[C1.]
			\item \textbf{$Q^\wt$ approximates the value gradient:}\\
			The weights learned by the approximate value function must satisfy $\wt = \argmin_{\wt'}\bell_{GE}(\theta,\wt')$, where 
			\begin{equation}
			 	\bell_{GE}(\theta,\wt)
				:= \expec\left[\left\|\grad_{\act} Q^{\wt}(\sta,\act)_{|\act=\bmu_\theta(\sta)} - \grad_{\act} Q^{\bmu}(\sta,\act)_{|\act=\bmu_\theta(\sta)}\right\|^2\right]
				\label{eq:val_estimate}
			\end{equation}
			is the mean-square difference between the gradient of the true value function $Q^{\bmu}$ and the approximation $Q^{\wt}$.
			\item \textbf{$Q^\wt$ is policy-compatible:}\\
			The gradients of the value-function and the policy must satisfy
			\begin{equation}
				\label{eq:grad_compat}
				\grad_{\act} Q^{\wt}(\sta,\act)_{|\act=\bmu_\theta(\sta)} = \big\langle \grad_\theta\bmu_\theta(\sta), \wt\big\rangle.
			\end{equation}			
		\end{enumerate}
\end{thm}

\begin{proof}
	See \citep{silver:14}.
\end{proof}

Having stated the compatibility condition, it is worth revisiting the problems that we propose to tackle in the paper. The first problem is to directly estimate the gradient of the value function, as required by Eq.~\eqref{eq:val_estimate} in condition \emph{C1}. The standard approach used in the literature is to estimate the value function, or the closely related advantage function, using temporal difference learning, and then compute the derivative of the estimate. The next section shows how the gradient can be estimated directly.

The second problem relates to the compatibility condition on policy and value gradients required by Eq.~\eqref{eq:grad_compat} in condition \emph{C2}. The only function approximation satisfying \emph{C2} that has been proposed is

\begin{eg}[standard value function approximation]\label{eg:advantage}\eod
	Let $\bphi(\sta)$ be an $m$-dimensional feature representation on states and set $\bphi(\sta,\act) := \grad_\theta\bmu_\theta(\sta)\cdot\big(\act-\bmu_\theta(\sta)\big)$.
	Then the value function approximation 
	\begin{equation}
		Q^{\vt,\wt}(\sta,\act) 
		= \underbrace{\left\langle \bphi(\sta,\act),\wt\right\rangle}_{\text{advantage function}} + \big\langle \bphi(\sta),\vt\big\rangle
		= (\act-\bmu_\theta(\sta))^\intercal \cdot\grad_\theta\bmu_\theta(\sta)^\intercal\cdot \wt +\bphi(\sta)^\intercal\cdot\vt.
	\end{equation}
	satisfies condition \emph{C2} of Theorem~\ref{t:compat}.
\end{eg}

The approximation in Example~\ref{eg:advantage} encounters serious problems when applied to \emph{deep} policies, see discussion in section~\ref{sec:problem}.


\section{Learning Value Gradients}
\label{sec:pbp}

In this section, we tackle the first problem by modifying temporal-difference (TD) learning so that it directly estimates the gradient of the value function. First, we developed a new approach to estimating the gradient of a black-box function at a point, based on perturbing the function with gaussian noise. It turns out that the approach extends easily to learning the gradient of a black-box function across its entire domain. Moreover, it is easy to combine with neural networks and temporal difference learning.

\subsection{Estimating the gradient of an unknown function at a point}

Gradient estimates have been intensively studied in bandit problems, where rewards (or losses) are observed but labels are not. Thus, in contrast to supervised learning where it is possible to compute the gradient of the loss, in bandit problems the gradient must be estimated. More formally, consider the following setup.

\begin{defn}[zeroth-order black-box]\label{def:bb}\eod
		A function $f:\bR^d\rightarrow \bR$ is a \textbf{zeroth-order black-box} if it can only be queried for \emph{zeroth-order} information. That is, User can request the value $f(x)$ of $f$ at any point $x\in\bR^d$, but cannot request the gradient of the function.

		We use the shorthand \emph{black-box} in what follows.
\end{defn}

The black-box model for optimization was introduced in \citep{nemirovski:83}, see \citep{raginsky:11} for a recent exposition. In those papers, a black-box consists in a \emph{first-order oracle} that can provide both zeroth-order information (the value of the function) and first-order information (the gradient or subgradient of the function).

\begin{rem}[reward function is a black-box; value function is not]\eod
	The reward function $r(\sta,\act)$ is a black box since Nature does not provide gradient information. The value function $Q^{\bmu_\theta}(\sta,\act) = \expec[r_1^\gamma | {\mathbf S}_1=\sta,{\mathbf A}_1=\act;\bmu_\theta]$ is \emph{not even} a black-box: it cannot be queried directly since it is defined as the expected discounted \emph{future} reward. It is for this reason the gradient perturbation trick must be combined with temporal difference learning, see section~\ref{sec:tdgl}.
\end{rem}

An important insight is that the gradient of an unknown function at a specific point can be estimated by perturbing its input \citep{flaxman:05}. For example, for small $\delta>0$ the gradient of $f:\bR^d\rightarrow \bR$ is approximately $\grad f(\x)_{|\x=\bmu}\approx d\cdot\expec_\bu[\frac{f(\bmu+\delta \bu)}{\delta} \bu]$ where the expectation is over vectors sampled uniformly from the unit sphere.

The following lemma provides a simple method for estimating the gradient of a function \emph{at a point} based on Gaussian perturbations:

\begin{lem}[gradient perturbation trick]\label{lem:gradient}\eod	
	The gradient of differentiable $f:\bR^d\rightarrow \bR$ at $\bmu\in\bR^d$ is
	\begin{equation}
		\label{e:grad}
		\grad_\x f(\x)_{|\x=\bmu} = \lim_{\sigma^2\rightarrow0}\argmin_{\wt\in\bR^d} \left\{\min_{b\in \bR}\expec_{\bepsilon\sim N(\bZe,\sigma^2\cdot \bI_d)}\left[\Big( f(\bmu+\bepsilon)-\langle\wt,\bepsilon\rangle - b\Big)^2\right]\right\}.
	\end{equation}
\end{lem}

\begin{proof}
	By taking sufficiently small variance, we can assume that $f$ is locally linear. Setting $b=f(\bmu)$ yields a line through the origin. It therefore suffices to consider the special case $f(\x)=\langle \vt,\x\rangle$. 

	Setting
	\begin{equation}
		\wt^* =\argmin_{\wt\in\bR^d}\expec_{\bepsilon\sim N(\bZe,\sigma^2\cdot\bI_d)} \left[\frac{1}{2}\Big(\langle\wt,\bepsilon\rangle-\langle\vt,\bepsilon\rangle\Big)^2\right],
	\end{equation}
	we are required to show that $\wt^*=\vt$. The problem is convex, so setting the gradient to zero requires to solve $0 = \expec\big[\langle\wt-\vt,\bepsilon\rangle\cdot \bepsilon\big]$, which reduces to solving the set of linear equations
	\begin{equation}
		\sum_{i=1}^d (w^i-v^i)\expec[\epsilon^i\epsilon^j] = (w^j-v^j)\expec[(\epsilon^j)^2]=(w^j-v^j)\cdot \sigma^2=0\qquad \text{ for all $j$}.
	\end{equation}
	The first equality holds since $\expec[\epsilon^i\epsilon^j]=0$. It follows immediately that $\wt^*=\vt$.
\end{proof}

\subsection{Learning gradients across a range}
\label{sec:learning_gradients}

The solution to the optimization problem in Eq.~\eqref{e:grad} is the gradient $\grad f(\x)$ of $f$ at a particular $\bmu\in\bR^d$. The next step is to learn a function $\Qt^\Wt:\bR^d\rightarrow \bR^d$ that approximates the gradient across a range of values. 

More precisely, given a sample $\{\x_i\}_{i=1}^n\sim \bP_X$ of points, we aim to find
\begin{equation}
	\Wt^* := \argmin_\Wt\sum_{i=1}^n\left[\left\|\grad f(\x_i)- \Qt^\Wt(\x_i)\right\|^2\right].
	\label{eq:grad_optimal}
\end{equation}
The next lemma considers the case where $Q^\vt$ and $\Qt^\Wt$ are linear estimates, of the form $Q^\vt(\x) := \langle\bphi(\x),\vt\rangle,$ and $\Qt^{\Wt}(\x)=\Wt\cdot \bpsi(\x)$ for fixed representations $\bphi:X\rightarrow\bR^m$ and $\bpsi:X\rightarrow\bR^n$.

\begin{lem}[gradient learning]\label{lem:gradl}\eod
	Let $f:\bR^d\rightarrow \bR$ be a differentiable function. Suppose that $\bphi:X\rightarrow \bR^m$ and $\bpsi:X\rightarrow\bR^n$ are representations such that there exists an $m$-vector $\vt^*$ and a $(d\times n)$-matrix $\Wt^*$ satisfying  $f(\x) = \langle\bphi(\x),\vt^*\rangle$ and $\grad f=\Wt^*\cdot \psi(\x)$ for all $\x$ in the sample.

	If we define loss function 
	\begin{equation}
 		\ell(\Wt,\Vt,\x,\sigma) = \expec_{\bepsilon} \left[\Big(f(\x+\bepsilon) - \langle \Qt^{\Wt}(\x),\bepsilon\rangle - Q^\Vt(\x)\Big)^2\right].	 	
	\end{equation}	
	then
	\begin{equation}
		\label{e:neural_grad}
 		\Wt^* = \lim_{\sigma^2\rightarrow 0}\argmin_{\Wt} \min_\Vt \expec_{\x\sim \hat{\bP}} \big[\ell(\Wt,\Vt,\x,\sigma)\big].	
	\end{equation}	
\end{lem}

\begin{proof}
	Follows from Lemma~\ref{lem:gradient}.
\end{proof}

In short, the lemma reduces gradient estimation to a simple optimization problem \emph{given a good enough representation}. Jumping ahead slightly to section~\ref{sec:dpg}, we ensure that our model has good enough representations by constructing two neural networks to learn them. The first neural network, $Q^\Vt:\bR^d\rightarrow\bR$, learns an approximation to $f(\x)$ that plays the role of the baseline $b$. The second neural network, $\Qt^\Wt:\bR^d\rightarrow\bR^d$ learns an approximation to the gradient.

\subsection{Temporal difference learning}

Recall that $Q^{\bmu}(\sta,\act)$ is the expected value of a state-action pair given policy $\bmu$. It is never observed directly, since it is computed by discounting over future rewards. TD-learning is a popular approach to estimating $Q^{\bmu}$ through dynamic programming \citep{sutton:98}. 

We quickly review TD-learning. Let $\bphi:\cS\times\cA\rightarrow\bR^m$ be a fixed representation. The goal is to find a value-estimate
\begin{equation}
	Q^\vt(\sta,\act) := \langle\bphi(\sta,\act),\vt\rangle,
\end{equation}
where $\vt$ is an $m$-dimensional vector, that is as close as possible to the true value function. If the value-function were known, we could simply minimize the mean-square error with respect to $\vt$:
\begin{equation}
	\label{e:MSE}
	\bell_{MSE}(\vt) = 
	\expec_{(\sta,\act)\sim (\rho^{\bmu}, \bmu)}\left[\Big(
	Q^{\bmu}(\sta,\act) - Q^\vt(\sta,\act) 
	\Big)^2\right].
\end{equation}
Unfortunately, it is impossible to minimize the mean-square error directly, since the value-function is the expected discounted future reward, rather than the reward. That is, the value function is not provided explicitly by the environment -- not even as a black-box. The Bellman error is therefore used a substitute for the mean-square error:
\begin{equation}
	\label{e:MSBE}
	\bell_{BE}(\vt) = 
	\expec_{(\sta,\act)\sim (\rho^{\bmu}, \bmu)}\Big[\Big(
	\overbrace{\underbrace{r(\sta,\act) + \gamma Q^{\vt}(\sta',\bmu(\sta'))}_{\approx Q^{\bmu}(\sta,\act)} - Q^\vt(\sta,\act)}^{\text{TD-error, }\delta}
	\Big)^2\Big]
\end{equation}
where $\sta'$ is the state subsequent to $\sta$. 

Let $\delta_t = r_t - Q^{\vt}(\sta_t,\act_t) + \gamma Q^{\vt}(\sta_{t+1},\bmu_\theta(\sta_{t+1}))$ be the TD-error. TD-learning updates $\vt$ according to
\begin{equation}
	\label{e:TD}
	\vt_{t+1}\leftarrow 
	\vt_t + \eta_t\cdot\delta_t \cdot \grad_\vt Q^\vt(\sta_t,\act_t) 
	= \vt_t + \eta_t\cdot\delta_t \cdot \bphi(\sta,\act),
\end{equation}
where $\eta_t$ is a sequence of learning rates. The convergence properties of TD-learning and related algorithms have been studied extensively, see \citep{tsitsiklis:97,dann:14}.

\subsection{Temporal difference gradient (TDG) learning}
\label{sec:tdgl}

Finally, we apply temporal difference methods to estimate the \emph{gradient}%
\footnote{Residual gradient (RG) and gradient temporal difference (GTD) methods were introduced in \citep{baird:95,sutton:09,sutton:09a}. The similar names may be confusing. RG and GTD methods are TD methods derived from gradient descent. In contrast, we develop a TD-based approach to \emph{learning gradients}. The two approaches are thus complementary and straightforward to combine. However, in this paper we restrict to extending vanilla TD to learning gradients.} 
of the value function, as required by condition \emph{C1} of Theorem~\ref{t:compat}. We are interested in gradient approximations of the form
\begin{equation}
	Q^\Wt(\sta,\act,\bepsilon) := \langle \Qt^{\Wt}(\sta,\act),\bepsilon\rangle= \langle\Wt\cdot \bpsi(\sta,\act),\bepsilon\rangle,
\end{equation}
where $\bpsi:\cS\times\cA\rightarrow\bR^n$ and $\Wt$ is a $(d\times n)$-dimensional matrix. The goal is to find $\Wt^*$ such that $\Qt^{\Wt^*}(\sta,\act)\approx \grad_{\bepsilon} Q^{\bmu}(\sta,\act,\bepsilon)_{|\bepsilon=\bZe} = \grad_{\act} Q^{\bmu}(\sta,\act)_{|\act=\bmu_\theta(\sta)}$ for all sampled state-action pairs. 

It is convenient to introduce notation $Q^{\bmu}(\sta,\act,\bepsilon):=Q^{\bmu}(\sta,\act+\bepsilon)$ and shorthand $\tilde{\sta}:=(\sta,\bmu_\Theta(\sta))$.  Then, analogously to the mean-square, define the perturbed gradient error: 
\begin{equation}
	\bell_{PGE}(\vt, \Wt;\sigma^2) 
	= \expec_{\sta\sim\rho^{\bmu}}\expec_{\bepsilon}\left[\Big(
	Q^{\bmu}(\tilde{\sta},\bepsilon) - \big\langle \Qt^\Wt(\tilde{\sta}),\bepsilon\big\rangle - Q^\vt(\tilde{\sta})
	\Big)^2\right],
\end{equation}
Given a good enough representation, Lemma~\ref{lem:gradl} guarantees that minimizing the perturbed gradient error yields the gradient of the value function. Unfortunately, as discussed above, the value function cannot be queried directly. We therefore introduce the Bellman gradient error as a proxy
\begin{equation}
	\bell_{BGE}(\vt,\Wt;\sigma^2) 
	= 
	\expec_{\sta\sim\rho^{\bmu}}\expec_{\bepsilon}\Big[\Big(
	\overbrace{
	\underbrace{r(\tilde{\sta},\bepsilon) + \gamma Q^\vt\big(\tilde{\sta}')}_{\approx Q^{\bmu}(\tilde{\sta},\bepsilon)}
	- \left\langle \Qt^\Wt(\tilde{\sta}),\bepsilon\right\rangle - Q^\vt(\tilde{\sta})
	}^{\text{TDG-error, }\tdg}
	\Big)^2\Big].
\end{equation}
Set the TDG-error as
\begin{equation}
	\tdg_t = r(\tilde{\sta}_t\bepsilon)+\gamma Q^\vt(\tilde{\sta}_{t+1}) - \langle \Qt^\Wt(\tilde{\sta}_t),\bepsilon\rangle - Q^\vt(\tilde{\sta}_t)
\end{equation}	
and, analogously to Eq.~\eqref{e:TD}, define the TDG-updates
\begin{align}
	\vt_{t+1} & 
	\leftarrow \vt_t + \eta_t\cdot\xi_t \cdot \grad_\vt Q^\vt(\tilde{\sta}_t) 
	= \vt_t + \eta_t\cdot\xi_t \cdot \bphi(\tilde{\sta}_t)\\
	\Wt_{t+1} & \leftarrow 
	\Wt_t + \eta_t\cdot\xi_t \cdot \grad_\Wt Q^{\Wt}(\tilde{\sta}_t)
	= \Wt_t + \eta_t\cdot\xi_t \cdot \bepsilon\otimes \bpsi(\tilde{\sta}_t),
\end{align}
where $\bepsilon\otimes \bpsi(\tilde{\sta}_t)$ is the $(d\times n)$ matrix given by the outer product. We refer to $\xi\cdot \bepsilon$ as the \textbf{perturbed TDG-error}. 

The following \emph{extension theorem} allows us to import guarantees from temporal-difference learning to temporal-difference gradient learning.

\begin{thm}[zeroth to first-order extension]\label{thm:extension}\eod
	Guarantees on TD-learning extend to TDG-learning.
\end{thm}

The idea is to reformulate TDG-learning as TD-learning, with a slightly different reward function and function approximation. Since the function approximation is still linear, any guarantees on convergence for TD-learning transfered automatically to TDG-learning.
\vspace{2mm}

\begin{proof}
	First, we incorporate $\bepsilon$ into the state-action pair. 
	Define $\tilde{r}(\sta,\act,\bepsilon):= r(\sta,\act+\bepsilon)$ and 
	\begin{equation}
		\tilde{\bpsi}(\sta,\act,\bepsilon) =  \bepsilon\otimes\bpsi(\sta,\act).
	\end{equation}	
	Second, we define a dot product on matrices of equal size by flattening them down to vectors. More precisely, given two matrices ${\mathbf A}$ and ${\mathbf B}$ of the same dimension $(m\times n)$, define the dot-product $\langle{\mathbf A},{\mathbf B}\rangle = \sum_{i,j=1}^{m,n}A_{ij}B_{ij}$. It is easy to see that
	\begin{equation}
		\Qt^\Wt(\sta,\act) := \langle\Wt\cdot\bpsi(\sta,\act),\bepsilon\rangle
		= \langle\tilde{\bpsi}(\sta,\act,\bepsilon),\Wt\rangle.
	\end{equation}
	The TDG-error can then be rewritten as	
	\begin{equation}
		\xi_t = \tilde{r}(\sta,\act,\bepsilon) + \gamma Q^{\vt,\Wt}(\sta',\act',\bepsilon') - Q^{\vt,\Wt}(\sta,\act,\bepsilon)
	\end{equation}
	where $Q^{\vt,\Wt}(\sta,\act,\bepsilon) = \langle\phi(\sta,\act),\vt\rangle + \langle\tilde{\bpsi}(\sta,\act,\bepsilon),\Wt\rangle$ is a linear function approximation. 

	If we are in a setting where TD-learning is guaranteed to converge to the value-function, it follows that TDG-learning is also guaranteed to converge -- since it is simply a different linear approximation. Thus, $Q^\mu(\tilde{\sta},\bepsilon)\approx Q^\vt(\tilde{\sta})+\Qt^\Wt(\tilde{\sta},\bepsilon)$ and the result follows by Lemma~\ref{lem:gradl}.
\end{proof}

\section{Algorithm: Value-Gradient Backpropagation}
\label{sec:dpg}

This section presents our model, which consists of three coupled neural networks that learn to estimate the value function, its gradient, and the optimal policy respectively. 

\begin{defn}[deviator-actor-critic]\label{def:beh_crit}\eod
	The \textbf{deviator-actor-critic (DAC)} model consists in three neural networks:
	\begin{itemize}
		\item \textbf{actor-network} with policy $\bmu_\Theta:\cS\rightarrow \cA\subset\bR^d$;
		\item \textbf{critic-network}, $Q^\Vt:\cS\times\cA\rightarrow \bR$, that estimates the value function; and 
		\item \textbf{deviator-network}, $\Qt^\Wt:\cS\times\cA\rightarrow\bR^d$, that estimates the gradient of the value function. 
	\end{itemize}
	Gaussian noise is added to the policy during training resulting in actions $\act=\bmu_\Theta(\sta)+ \bepsilon$ where $\bepsilon\sim N(\bZe,\sigma^2\cdot\bI_d)$. The outputs of the critic and deviator are combined as
	\begin{equation}
		Q^{\Wt,\Vt}\Big(\sta,\bmu_\Theta(\sta),\bepsilon\Big) 
		= Q^\Vt\big(\sta,\bmu_\Theta(\sta)\big) 
		+ \Big\langle \Qt^\Wt\big(\sta,\bmu_\Theta(\sta)\big),\bepsilon\Big\rangle.
	\end{equation}	
\end{defn}

The Gaussian noise plays two roles. Firstly, it controls the explore/exploit tradeoff by controlling the extent to which Actor deviates from its current optimal policy. Secondly, it controls the ``resolution'' at which Deviator estimates the gradient. 

The three networks are trained by backpropagating three different signals. Critic, Deviator and Actor backpropagate the TDG-error, the perturbed TDG-error, and Deviator's gradient estimate respectively; see Algorithm~\ref{alg:qprop}. An explicit description of the weight updates of individual units is provided in Appendix~\ref{sec:explicit}.

Deviator estimates the gradient of the value-function with respect to \emph{deviations $\bepsilon$ from the current policy}. Backpropagating the gradient through Actor allows to estimate the influence of Actor-parameters on the value function as a function of their effect on the policy.
\vspace{2mm}

\begin{algorithm}[H]
	\caption{\texttt{Value-Gradient Backpropagation ($\qprop$)}.\label{alg:qprop}}
	\DontPrintSemicolon
	\SetKwInOut{Input}{input}
 	\For{rounds $t =1, 2, \ldots, T$}{
		Network gets state $\sta_t$, responds $\act_t=\bmu_{\Theta_t}(\sta_t)+\bepsilon$, gets reward $r_t$\; 
		Let $\tilde{\sta} := (\sta,\bmu_\Theta(\sta))$.
		\;
		$\tdg_t \longleftarrow r_t  + \gamma Q^{\Vt_t}(\tilde{\sta}_{t+1})
		- Q^{\Vt_t}(\tilde{\sta}_t) - \big\langle \Qt^{\Wt_t}(\tilde{\sta}_t),\bepsilon\big\rangle$
		\qquad\quad\texttt{// compute TDG-error}
		\;
		$\Theta_{t+1} 
			\longleftarrow \Theta_{t} 
			+ \eta^A_t \cdot\grad_{\Theta}\bmu_{\Theta_t}(\sta_t)\cdot \Qt^{\Wt_t}\big(\tilde{\sta}_t\big)$
			\quad\quad\qquad\qquad\texttt{// backpropagate $\mathtt \Qt^\Wt$}
		\;
		$\Vt_{t+1} 
			\longleftarrow \Vt_t+ \eta^C_t\cdot \tdg_t\cdot \grad_\Vt Q^{\Vt_t}(\tilde{\sta}_t)$
			\qquad\qquad\qquad\qquad\qquad\texttt{// backpropagate $\tdg$}
		\;
		$\Wt_{t+1} 
			\longleftarrow \Wt_t+ \eta^D_t \cdot\tdg_t\cdot\grad_\Wt \Qt^{\Wt_t}(\tilde{\sta}_t)\cdot \bepsilon$
			\quad\qquad\qquad\quad\quad\texttt{// backpropagate $\tdg\cdot\bepsilon$}
		\;
	}
\end{algorithm}
\vspace{2mm}

Critic and Deviator learn representations suited to estimating the value function and its gradient respectively. Note that even though the gradient is a linear function \emph{at a point}, it can be a highly nonlinear function in general. Similarly, Actor learns a \emph{policy} representation. 

We set the learning rates of Critic and Deviator to be equal $(\eta^C_t=\eta^D_t)$ in the experiments in section~\ref{sec:experiments}. However, the perturbation $\bepsilon$ has the effect of slowing down and stabilizing Deviator updates:

\begin{rem}[stability]\label{rem:stability}\eod
	The magnitude of Deviator's weight updates depend on $\bepsilon\sim N(\bZe,\sigma^2\cdot\bI_d)$ since they are computed by backpropagating the perturbed TDG-error $\xi\cdot\bepsilon$. Thus as $\sigma^2\rightarrow0$, Deviator's learning rate essentially tends to zero. In general, Deviator learns more slowly than Critic. This has a stabilizing effect on the policy since Actor is insulated from Critic -- its weight updates only depend (directly) on the output of Deviator. 
\end{rem}
\section{Analysis: Deep Compatible Function Approximation}
\label{sec:thm_main}

Our main result is that the deviator's value gradient is compatible with the policy gradient of each unit in the actor-network -- considered as an actor in its own right:

\begin{thm}[deep compatible function approximation]\label{thm:main}\eod
	Suppose that all units are rectilinear or linear. Then for each Actor-unit in the Actor-network there exists a reparametrization of the value-gradient approximator, $\Qt^\Wt$, that satisfies the compatibility conditions in Theorem~\ref{t:compat}.
\end{thm}

The actor-network is thus a collection of interdependent agents that individually follow the correct policy gradients. The experiments below show that they also collectively converge on useful behaviors.

\paragraph{Overview of the proof.}
The next few subsections prove Theorem~\ref{thm:main}. We provide a brief overview before diving into the details.

Guarantees for temporal difference learning and policy gradients are typically based on the assumption that the value-function approximation is a \emph{linear} function of the learned parameters. However, we are interested in the case where Actor, Critic and Deviator are all neural networks, and are therefore highly nonlinear functions of their parameters. The goal is thus to relate the representations learned by neural networks to the prior work on linear function approximations.

To do so, we build on the following observation, implicit in \citep{srivastava:14}:

\begin{rem}[active submodels]\label{rem:active}\eod
	A neural network of $n$ linear and rectilinear units can be considered as a set of $2^n$ submodels, corresponding to different subsets of units. The active submodel at time $t$ consists in the active units (that is, the linear units and the rectifiers that do not output 0). 

	The active submodel has two important properties:
	\begin{itemize}
		\item it is a \emph{linear} function from inputs to outputs, since rectifiers are linear when active, and 
		\item at each time step, learning only occurs over the active submodels, since only active units update their weights.
	\end{itemize}
\end{rem}
The feedforward sweep of a rectifier network can thus be disentangled into two steps \citep{doco:15}. The first step, which is highly nonlinear, applies a gating operation that selects the active submodel -- by rendering various units inactive. The second step computes the output of the neural network via matrix multiplication. It is important to emphasize that although the  active submodel is a linear function from inputs to outputs, it is not a linear function of the weights. 

The strategy of the proof is to decompose the Actor-network in an interacting collection of agents, referred to as Actor-units. That is, we model each unit in the Actor-network as an Actor in its own right that. On each time step that an Actor-unit is active, it interacts with the Deviator-submodel corresponding to the current active submodel of the Deviator-network. The proof shows that each Actor-unit has compatible function approximation.

\subsection{Error backpropagation on rectilinear neural networks}

First, we recall some basic facts about backpropagation in the case of \emph{rectilinear} units.  Recent work has shown that replacing sigmoid functions with rectifiers $S(x)=\max(0,x)$ improves the performance of neural networks \citep{nair:10,glorot:11,zeiler:13,dahl:13}. 

Let us establish some notation. The output of a rectifier with weight vector $\wt$ is 
\begin{equation}
	S_{\wt}(\x) := S(\langle\wt,\x\rangle):=\max(0,\langle\wt,\x\rangle).
\end{equation}
The rectifier is \textbf{active} if $\langle\wt,\x\rangle>0$. We use rectifiers because they perform well in practice and have the nice property that units are \emph{linear} when they are active. 
The rectifier subgradient is the indicator function
\begin{equation*}
	\bbO(x):= \grad S(x) = \begin{cases}
		1 & x > 0\\
		0 & \text{else}.
	\end{cases}	
\end{equation*}

Consider a neural network of $n$ units, each equipped with a weight vector $\wt^j\in\ccH_j\subset\bR^{d_j}$. Hidden units are rectifiers; output units are linear. There are $n$ units in total. It is convenient to combine all the weight vectors into a single object; let $\Wt\subset\ccH=\prod_{j=1}^n\ccH_j \subset\bR^N$ where $N=\sum_{j=1}^nd_j$. The network is a function $F^\Wt:\bR^m\rightarrow \bR^d:\x_{\text{in}}\mapsto F^\Wt(\x_{\text{in}})=:\x_{\text{out}}$. 

The network has error function $\error(\x_\text{out},\y)$ with gradient $\subg = \grad_{\x_\text{out}}\error$. Let $x^j$ denote the output of unit $j$ and $\bphi^j(\x_{\text{in}})=(x^i)_{\{i:i\rightarrow j\}}$ denote its input, so that $x^j = S(\langle\wt^j,\bphi^j(\x_{\text{in}})\rangle$. Note that $\bphi^j$ depends on $\Wt$ (specifically, the weights of lower units) but this is supressed from the notation.

\begin{defn}[influence]\eod
	The \textbf{\emph{influence}} of unit $j$ on unit $k$ at time $t$ is $\pi^{j,k}_t := \frac{\dd x^k_t}{\dd x^j_t}$ \citep{bvb:15}. The influence of unit $j$ on the output layer is the vector $\bpi^{j}_t = \big(\pi^{j,k}_t\big)_{k\in \text{out}}$.
\end{defn}

The following lemma summarizes an analysis of the feedforward and feedback sweep of neural nets.
\begin{lem}[structure of neural network gradients]\label{lem:structure}\eod
	The following properties hold
	\begin{enumerate}[a.]
		\item \textbf{Influence.}\\
		A path is \textbf{\emph{active}} at time $t$ if all units on the path are firing. The influence of $j$ on $k$ is the sum of products of weights over all active paths from $j$ to $k$:
		\begin{equation}
			\pi^{j,k}_t = 
			\sum_{\{\alpha|j\rightarrow \alpha\}} w^{j,\alpha}\bbO^\alpha_t
			\left(\sum_{\{\beta|\alpha\rightarrow \beta\}}w^{\alpha,\beta}\bbO^\beta_t\left( \cdots \sum_{\{\omega|\omega\rightarrow k\}}w^{\omega,k}\bbO^k_t\right)\right).
		\end{equation}
		where $\alpha,\beta,\ldots,\omega$ refer to units along the path from $j$ to $k$.
		\item \textbf{Output decomposition.}\\
		The output of a neural network decomposes, relative to the output of unit $j$, as
		\begin{equation}
			F^\Wt(\x_{\text{in}}) = \bpi^{j}\cdot x^j + \bpi^{-j} \cdot \x_{\text{in}},
		\end{equation}
		where $\bpi^{-j}$ is the $(m\times d)$-matrix whose $(ik)^\text{th}$ entry is the sum over all active paths from input unit $i$ to output unit $k$ that do not intersect unit $j$.
		\item \textbf{Output gradient.}\\
		Fix an input $\x_{\text{in}}\in\bR^m$ and consider the network as a function from parameters to outputs $F^\bullet(\x_{\text{in}}):\ccH\rightarrow \bR^d:\Wt\mapsto F^\Wt(\x_{\text{in}})$ whose gradient is an $(N\times d)$-matrix. The $(ij)^\text{th}$-entry of the gradient is the input to the unit times its influence:
		\begin{equation}
			\Big(\nabla_\Wt F^\Wt(\x_{\text{in}})\Big)_{ij} = \begin{cases}
				\phi^{ij}(\x_{\text{in}})\cdot \bpi^{j} & \text{if unit $j$ is active} \\
				0 & \text{else.}
			\end{cases}			
		\end{equation}
		\item \textbf{Backpropagated error.}\\
		Fix $\x_{\text{in}}\in\bR^m$ and consider the function $\error(\Wt)= \error(F^\bullet(\x_{\text{in}}),\y):\ccH\rightarrow \bR:\Wt\mapsto \error(F^\Wt(\x_{\text{in}}),\y)$. Let $\subg = \grad_{\x_{out}}\error(\x_{out},y)$.

		The gradient of the error function is
		\begin{align}
			\label{eq:backprop}
			\left(\nabla_\Wt \error\right)_{ij} 
			& = \left\langle\subg, \left(\nabla_\Wt F^\Wt(\x_{\text{in}})\right)_{ij}\right\rangle\\
			& = \subg^\intercal\cdot \left(\nabla_\Wt F^\Wt(\x_{\text{in}})\right)_{ij}
			= \delta^j \cdot \bphi^j(\x_{\text{in}})
		\end{align}
		where the backpropagated error signal $\delta^j$ received by unit $j$ decomposes as $\delta^j = \left\langle \subg,\bpi^{j} \right\rangle$.
	\end{enumerate}	
\end{lem}

\begin{proof}
	Direct computation.
\end{proof}

The lemma holds generically for networks of rectifier and linear units. We apply it to actor, critic and deviator networks below.

\subsection{A minimal DAC model}

This subsection proves condition \emph{C1} of compatible function approximation for a minimal, linear Deviator-Actor-Critic model. The next subsection shows how the minimal model arises at the level of Actor-units.

\begin{defn}[minimal model]\eod
	The \textbf{minimal model} of a Deviator-Actor-Critic consists in an Actor with linear policy $\mu_\theta(\sta) = \langle\theta, \bphi(\sta)\rangle + \epsilon$, 
	where $\theta$ is an $m$-vector and $\epsilon$ is a noisy scalar. The Critic and Deviator together output:
	\begin{equation}
		Q^{w,\vt}(\sta,\mu_\theta(\sta),\epsilon) 
		= Q^{\vt}(\sta) + G^w(\mu_\theta(\sta),\epsilon)
		= \underbrace{\langle\bphi(\sta),\vt\rangle}_{\text{Critic}}
		+ \underbrace{\mu_\theta(\sta)\cdot\langle \epsilon,w \rangle}_{\text{Deviator}},
	\end{equation}
	where $\vt$ is an $m$-vector, $w$ is a scalar, and $\langle \epsilon,w\rangle$ is simply scalar multiplication.	
\end{defn}

The Critic in the minimal model is standard. However, the Deviator has been reduced to almost nothing: it learns a single scalar parameter, $w$, that is used to train the actor. The minimal model is thus too simple to be much use as a standalone algorithm. 

\begin{lem}[compatible function approximation for the minimal model]\label{lem:loc_compat}\eod	
	There exists a reparametrization of the gradient estimate of the minimal model $G^{\tilde{\wt}}(\sta,\epsilon) = G^w(\mu_\theta(\sta),\epsilon)$ such that compatibility condition \emph{C1} in Theorem~\ref{t:compat} is satisifed:
	\begin{equation}
		\grad_\epsilon G^{\tilde{\wt}}(\sta,\epsilon) 
		= \langle\grad_\theta \mu_\theta(\sta),\tilde{\wt}\rangle.
	\end{equation}
\end{lem}

\begin{proof}	
	Let $\tilde{\wt} := w\cdot \theta^\intercal$ and construct  $G^{\tilde{\wt}}(\sta,\epsilon):= \langle\tilde{\wt}\cdot \bphi(\sta),\epsilon\rangle$.
	Clearly,
	\begin{equation}
		G^{\tilde{\wt}}(\sta,\epsilon)
		= \langle w\cdot \theta^\intercal\cdot \bphi(\sta), \epsilon\rangle 
		=\mu_\theta(\sta)\cdot\langle w,\epsilon \rangle
		= G^{w}(\mu_\theta(\sta),\epsilon).
	\end{equation}
	Observe that $\grad_\epsilon G^{\tilde{\wt}}(\sta,\epsilon) = w\cdot \mu_\theta(\sta)$ and that, similarly,  
	\begin{equation}
		\big\langle \grad_\theta \mu_\theta(\sta), \tilde{\wt}\big\rangle 
		= w\cdot\mu_\theta(\sta) 
	\end{equation}
	as required.
\end{proof}

\subsection{Proof of Theorem~\ref{thm:main}}

The proof proceeds by showing that the compatibility conditions in Theorem~\ref{t:compat} hold for each Actor-unit. The key step is to relate the Actor-units to the minimal model introduced above.

\begin{lem}[reduction to minimal model]\label{lem:reduction}\eod
	Actor-units in a DAC neural network are equivalent to minimal model Actors.
\end{lem}

\begin{proof}
	Let $\bpi^j_t$ denote the influence of unit $j$ on the output layer of the Actor-network at time $t$. When unit $j$ is active, Lemma~\ref{lem:structure}ab implies we can write $\bmu_{\Theta_t}(\sta_t)=\bpi^j_t\cdot \langle\theta^j_t,\bphi^j_t(\sta_t)\rangle + \bmu_{\Theta^{-j}_t}(\sta_t)$, where $\bmu_{\Theta^{-j}_t}(\sta_t)$ is the sum over all active paths from the input to the output of the Actor-network that do not intersect unit $j$. 

	Following Remark~\ref{rem:active}, the active subnetwork of the Deviator-network at time $t$ is a linear transform which, by abuse of notation, we denote by $\Wt'_t$. 
	
	Combine the last two points to obtain
	\begin{align}
		\Qt^{\Wt_t}(\tilde{\sta}_t) & = \Wt'_t\cdot \left(\bpi^j_t\cdot \langle\theta^j,\bphi^j_t(\sta_t)\rangle + \bmu_{\Theta^{-j}_t}(\sta_t)\right) \\
		& = (\Wt'_t\cdot \bpi^j_t)\cdot \langle\theta^j,\bphi^j_t(\sta_t)\rangle + \text{terms that can be omitted}.
	\end{align}
	Observe that $(\Wt'_t\cdot \bpi^j_t)$ is a $d$-vector. We have therefore reduced Actor-unit $j$'s interaction with the Deviator-network to $d$ copies of the minimal model.
\end{proof}

\noindent
Theorem~\ref{thm:main} follows from combining the above Lemmas.

\vspace{2mm}
\begin{proof}
	Compatibility condition \emph{C1} follows from Lemmas~\ref{lem:loc_compat} and \ref{lem:reduction}.	
	Compatibility condition \emph{C2} holds since the Critic and Deviator minimize the Bellman gradient error with respect to $\Wt$ and $\Vt$ which also, implicitly, minimizes the Bellman gradient error with respect to the corresponding reparametrized $\tilde{\wt}$'s for each Actor-unit. 
\end{proof}

Theorem~\ref{thm:main} shows that each Actor-unit satisfies the conditions for compatible function approximation and so follows the correct gradient when performing weight updates. 


\subsection{Structural credit assignment for multiagent learning}
\label{sec:local_actors}

It is interesting to relate our approach to the literature on multiagent reinforcement learning \citep{guestrin:02,agogino:04,agogino:08}. In particular, \citep{holmesparker:14} consider the \emph{structural credit assignment problem} within populations of interacting agents: How to reward individual agents in a population for rewards based on their collective behavior? They propose to train agents within populations with a \emph{difference-based objective} of the form 
\begin{equation}
	\label{eq:tumer_diff}
	D_j = Q(\z) - Q(\z_{-j}, {\mathbf c}_j)
\end{equation}
where $Q$ is the objective function to be maximized; $\z_j$ and $\z_{-j}$ are the system variables that are and are not under the control of agent $j$ respective, and ${\mathbf c}_j$ is a fixed counterfactual action.

In our setting, the gradient used by Actor-unit $j$ to update its weights can be described explicitly:

\begin{lem}[local policy gradients]\eod
	Actor-unit $j$ follows policy gradient 
	\begin{equation}
		\grad J[\bmu_{\theta^j}] = \expec\left[\grad_{\theta^j} \bmu_{\theta^j}(s)\cdot \big\langle \bpi^j,\Qt^\Wt(\tilde{\sta})\big\rangle\right],
	\end{equation}		
	where $\langle \bpi^j,\Qt^\Wt(\tilde{\sta})\big\rangle\approx D_{\bpi^j}Q^{\bmu}(\tilde{\sta})$ is Deviator's estimate of the directional derivative of the value function in the direction of Actor-unit $j$'s influence.
\end{lem}

\begin{proof}
	Follows from Lemma~\ref{lem:structure}b.
\end{proof}

Notice that $\grad_{\z_j} Q= \grad_{\z_j} D_j$ in Eq.~\eqref{eq:tumer_diff}. It follows that training the Actor-network via $\qprop$ causes the Actor-units to optimize the difference-based objective -- without requiring to compute the difference explicitly. Although the topic is beyond the scope of the current paper, it is worth exploring how suitably adapted variants of backpropagation can be applied to the reinforcement learning problems in the multiagent setting.

\subsection{Comparison with related work}
\label{sec:problem}

\paragraph{Comparison with $\copdac$.}
Extending the standard value function approximation in Example~\ref{eg:advantage} to the setting where Actor is a neural network yields the following representation, which is used in \citep{silver:14} when applying $\copdac$ to the octopus arm task:

\begin{eg}[extension of standard value approximation to neural networks]\label{eg:deep_advantage}\eod
	Let $\bmu_\Theta:\cS\rightarrow \cA$ and $Q^\Vt:\cS\rightarrow\bR$ be an Actor and Critic neural network respectively. Suppose the Actor-network has $N$ parameters (i.e. the total number of entries in $\Theta$). It follows that the Jacobian $\grad_\Theta\bmu_\Theta(\sta)$ is an $(N\times d)$-matrix. 

	The value function approximation is then
	\begin{equation}
		Q^{\Vt,\Wt}(\sta,\act) 
		= \underbrace{(\act-\bmu_\Theta(\sta))^\intercal \cdot\nabla_\Theta\bmu_\Theta(\sta)^\intercal\cdot \wt}_{\text{advantage function}}  + \underbrace{Q^{\Vt}(\sta)}_{\text{Critic}}.
	\end{equation}
	where $\wt$ is an $N$-vector.
\end{eg}

Weight updates under $\copdac$, with the function approximation above, are therefore as described in Algorithm~\ref{alg:copdac}.
\vspace{2mm}

\begin{algorithm}[H]
	\caption{\texttt{Compatible Deterministic Actor-Critic ($\copdac$)}.\label{alg:copdac}}
	\DontPrintSemicolon
	\SetKwInOut{Input}{input}
 	\For{rounds $t =1, 2, \ldots, T$}{
		Network gets state $\sta_t$, responds $\act_t=\bmu_{\Theta_t}(\sta_t)+\bepsilon$ where $\bepsilon\sim N(\bZe,\sigma^2\cdot \bI_d)$, gets reward $r_t$
		\;
		$\delta_t \longleftarrow r_t  
			+ \gamma Q^{\Vt_t}(\sta_{t+1})
			- Q^{\Vt_t}(\sta_t) - \langle \grad_\Theta \bmu_{\Theta_t}(\sta_t)\cdot \bepsilon,\wt_t\rangle$
		\;
		$\Theta_{t+1} 
			\longleftarrow \Theta_{t} 
			+ \eta^A_t \cdot\grad_\Theta\bmu_{\Theta_t}(\sta_t)\cdot \grad_\Theta\bmu_{\Theta_t}(\sta_t)^\intercal\cdot \wt_t$
		\;
		$\Vt_{t+1} 
			\longleftarrow \Vt_t+ \eta^C_t \cdot\delta_t\cdot\grad_\Vt Q^{\Vt_t}(\sta_t)$
		\;
		$\wt_{t+1} 
			\longleftarrow \wt_t
			+ \eta^C_t \cdot\delta_t\cdot \grad_\Theta\bmu_{\Theta_t}(\sta_t)\cdot \bepsilon$
		\;
	}
\end{algorithm}
\vspace{2mm}

\noindent
Let us compare $\qprop$ with $\copdac$, considering the three updates in turn:
\begin{itemize}
	\item \emph{Actor updates.}\\
	Under $\qprop$, the  Actor backpropagates the value-gradient estimate. In contrast under $\copdac$ the Actor performs a complicated update that combines the policy gradient $\nabla_\Theta\bmu(\sta)$ with the advantage function's weights -- and differs substantively from backprop. 

	\item \emph{Deviator / advantage-function updates.}\\
	Under $\qprop$, the Deviator backpropagates the perturbed TDG-error. In contrast,  $\copdac$ uses the gradient of the \emph{Actor} to update the weight vector $\wt$ of the advantage function.

	By Lemma~\ref{lem:structure}d, backprop takes the form $\subg^\intercal \cdot \grad_\Theta\bmu_\Theta(\sta)$ where $\subg$ is a $d$-vector. In contrast, the advantage function requires computing $\grad_\Theta\bmu_\Theta(\sta)^\intercal\cdot\wt$, where $\wt$ is an $N$-vector. Although the two formulae appear similarly superficially, they carry very different computational costs. 

	The first consequence is that the parameters of $\wt$ must exactly line up with those of the policy. The second consequence is that, by Lemma~\ref{lem:structure}c, the advantage function requires access to
	\begin{equation}
		\left(\nabla_\Theta \bmu_\Theta(\sta)\right)_{ij} = \begin{cases}
			\bphi^{ij}(\sta)\cdot \bpi^j & \text{if unit $j$ is active}\\
			0 & \text{else},
		\end{cases}
	\end{equation}
	where $\bphi^{ij}(\sta)$ is the input from unit $i$ to unit $j$. Thus, the advantage function  requires access to the input $\bphi^j(\sta)$ and the influence $\bpi^{j}$ of every unit in the Actor-network.
	
	\item \emph{Critic updates.}\\
	The critic updates for the two algorithms are essentially identical, with the TD-error replaced with the TDG-error.
\end{itemize}

In short, the approximation in Example~\ref{eg:deep_advantage} that is used by $\copdac$ is thus not well-adapted to deep learning. The main reason is that learning the advantage function requires coupling the vector $\wt$ with the parameters $\Theta$ of the actor.

\paragraph{Comparison with computing the gradient of the value-function approximation.}
Perhaps the most natural approach to estimating the gradient is to simply estimate the value function, and then use its gradient as an estimate of the derivative \citep{jordan:90,prokhorov:97,wang:01,hafner:11,fairbank:12,fairbank:13}. The main problem with this approach is that, to date, it has not been show that the resulting updates of the Critic and the Actor are compatible. 

There are also no guarantees that the gradient of the Critic will be a good approximation to the gradient of the value function -- although it is intuitively plausible. The problem becomes particularly severe when the value-function is estimated via a neural network that uses activation functions that are \emph{not smooth} such as rectifers. Rectifiers are becoming increasingly popular due to their superior empirical performance \citep{nair:10,glorot:11,zeiler:13,dahl:13}.

\section{Experiments}
\label{sec:experiments}

We evaluate $\qprop$ on three tasks: two highly nonlinear contextual bandit tasks constructed from benchmark datasets for nonparametric regression, and the octopus arm. 

We do not evaluate $\qprop$ on other standard reinforcement learning benchmarks such as Mountain Car, Pendulum or Puddle World, since these can already be handled by \emph{linear} actor-critic algorithms. The contribution of $\qprop$ is the ability to learn representations suited to nonlinear problems.

\paragraph{Cloning and replay.}
Temporal difference learning can be unstable when run over a neural network. A recent innovation introduced in \citep{Mnih:2015wq} that \emph{stabilizes} TD-learning is to clone a separate network $Q^{\tilde{\Vt}}$ to compute the targets $r_t + \gamma Q^{\tilde{\Vt}}(\tilde{\sta}_{t+1})$. The parameters of the cloned network are updated periodically. 

We implement a similar modification of the TDG-error in Algorithm~\ref{alg:qprop}. We also use experience replay \citep{Mnih:2015wq}. $\qprop$ is well-suited to replay, since the critic and deviator can learn values and gradients over the full range of previously observed state-action pairs offline. 

Cloning and replay were also applied to $\copdac$. Both algorithms were implemented in Theano \citep{bergstra:10,bastien:12}.

\subsection{Contextual Bandit Tasks}
\label{sec:cb}

\begin{figure}[!htb]
	\centering
	\begin{center}
		\subfigure{\includegraphics[width=2.85in]{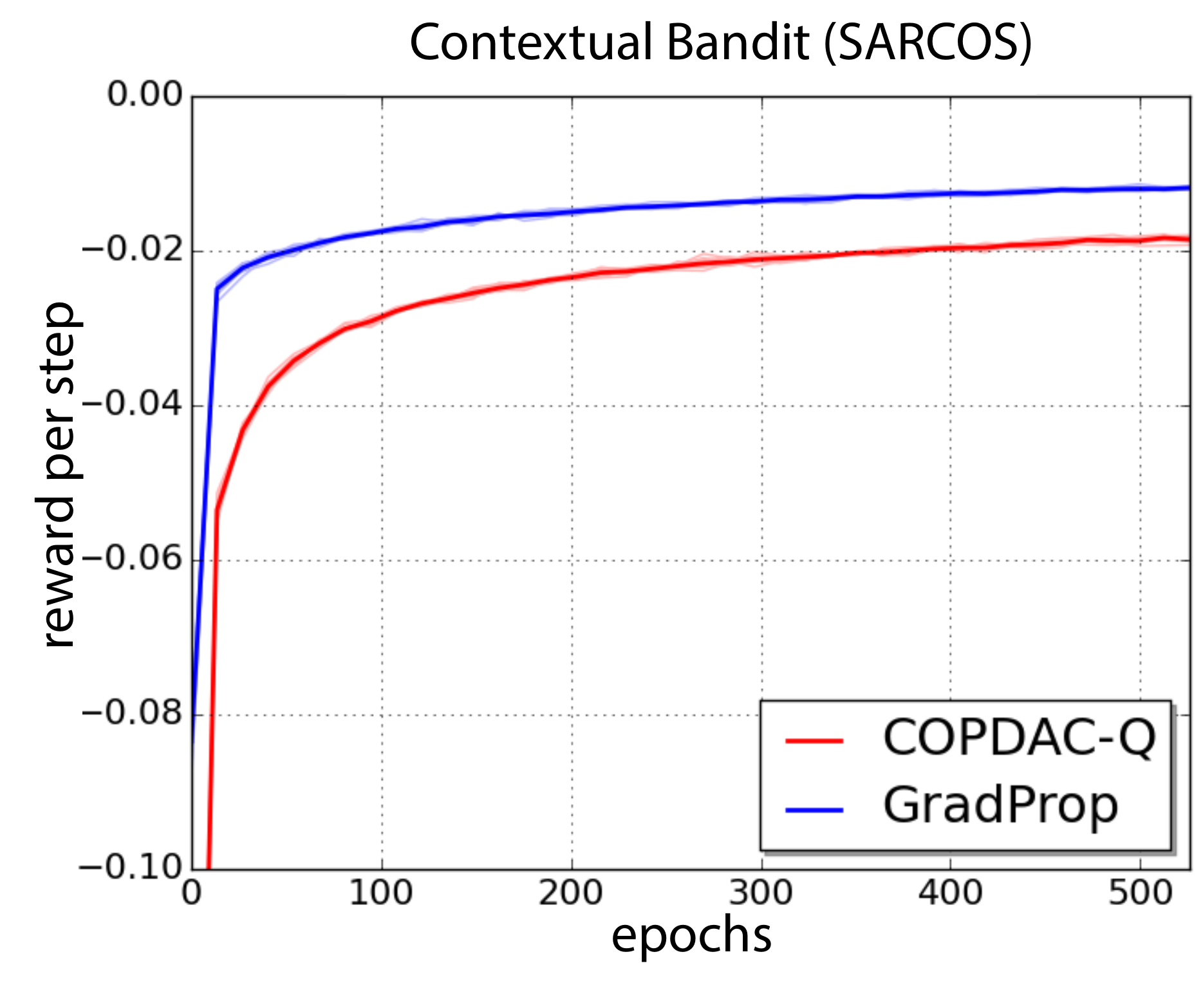}}
		\hspace{5mm}
		\subfigure{\includegraphics[width=2.85in]{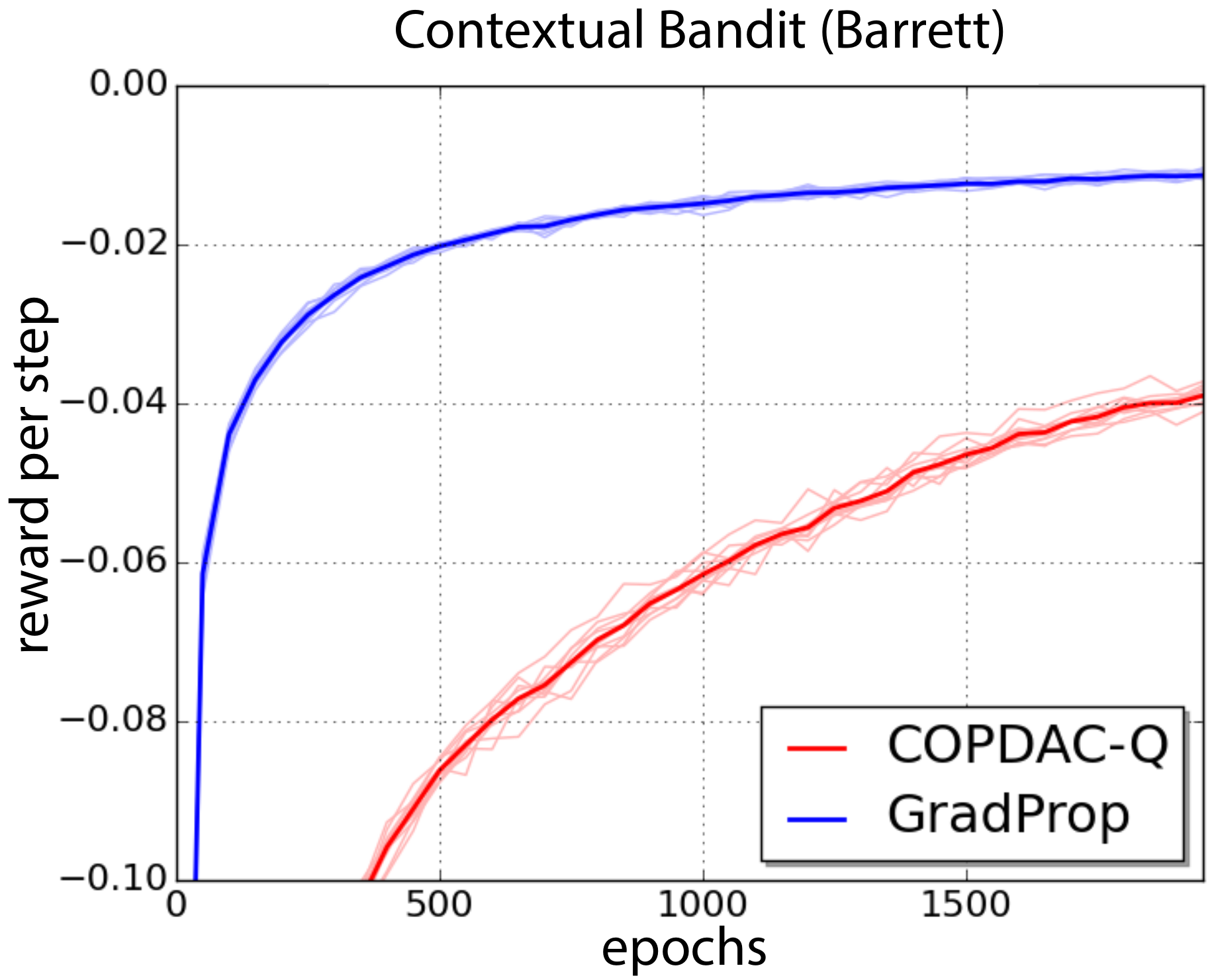}}
	\end{center}
	\vspace{-3mm}
	\caption{\textbf{Performance on contextual bandit tasks.}
	The reward (negative normalized test MSE) for 10 runs are shown and averaged (thick lines). Performance variation for $\qprop$ is barely visible. Epochs refer to multiples of dataset; algorithms are ultimately trained on the same number of random samples for both datasets.
	}
	\label{f:robo}
\end{figure}

The goal of the contextual bandit tasks is to probe the ability of reinforcement learning algorithms to accurately estimate gradients. The experimental setting may thus be of independent interest.

\paragraph{Description.}
We converted two robotics datasets, SARCOS\footnote{Taken from \texttt{www.gaussianprocess.org/gpml/data/}.} and Barrett WAM\footnote{Taken from \texttt{http://www.ausy.tu-darmstadt.de/Miscellaneous/Miscellaneous}.}, into contextual bandit problems via the supervised-to-contextual-bandit transform in \citep{dudik:14}. The datasets have 44,484 and 12,000 training points respectively, both with 21 features corresponding to the positions, velocities and accelerations of seven joints. Labels are 7-dimensional vectors corresponding to the torques of the 7 joints. 

In the contextual bandit task, the agent samples 21-dimensional state vectors i.i.d. from either the SARCOS or Barrett training data and executes 7-dimensional actions. The reward $r(\sta,\act)= -\|\y(\sta)-\act\|_2^2$ is the negative mean-square distance from the action to the label. Note that the reward is a scalar, whereas the correct label is a 7-dimensional vector. The gradient of the reward 
\begin{equation}
	\frac{1}{2}\grad_\act r(\sta,\act) = \y(\sta)-\act
\end{equation}
is the direction from the action to the correct label. In the supervised setting, the gradient can be computed. In the bandit setting, the reward is a zeroth-order black box.

The agent thus receives far less information in the bandit setting than in the fully supervised setting. Intuitively, the negative distance $r(\sta,\act)$ ``tells'' the algorithm that the correct label lies on the surface of a sphere in the 7-dimensional action space that is centred on the most recent action. By contrast, in the supervised setting, the algorithm is given the position of the label in the action space. In the bandit setting, the algorithm must estimate the position of the label on the surface of the sphere. Equivalently, the algorithm must estimate the label's direction relative to the center of the sphere -- which is given by the \emph{gradient} of the value function.

The goal of the contextual bandit task is thus to \emph{simultaneously solve seven nonparametric regression problems when observing distances-to-labels instead of directly observing labels}. The value function is relatively easy to learn in contextual bandit setting since the task is not sequential. However, both the value function and its gradient are highly nonlinear, and it is precisely the gradient that specifies where labels lie on the spheres.

\paragraph{Network architectures.}
$\qprop$ and $\copdac$ were implemented on an actor and deviator network of two layers (300 and 100 rectifiers) each and a critic with a hidden layers of 100 and 10 rectifiers. Updates were computed via RMSProp with momentum. The variance of the Gaussian noise $\sigma$ was set to decrease linearly from $\sigma^2=1.0$ until reaching $\sigma^2=0.1$ at which point it remained fixed.

\paragraph{Performance.}
Figure~\ref{f:robo} compares the test-set performance of policies learned by $\qprop$ against $\copdac$. The final policies trained by $\qprop$ achieved average mean-square test error of 0.013 and 0.014 on the seven SARCOS and Barrett benchmarks respectively. 

Remarkably, $\qprop$ is competitive with fully-supervised nonparametric regression algorithms on the SARCOS and Barrett datasets, see Figure~2bc in \citep{nguyen:08} and the results in \citep{kpotufe:13,trivedi:14}. It is important to note that the results reported in those papers are for algorithms that are given the labels and that solve \emph{one regression problem at a time}. To the best of our knowledge, there are no prior examples of a bandit or reinforcement learning algorithm that is competitive with fully supervised methods on regression datasets.

For comparison, we implemented Backprop on the Actor-network under full-supervision. Backprop converged to .006 and .005 on SARCOS and BARRETT, compared to 0.013 and 0.014 for $\qprop$. Note that BackProp is trained on 7-dim labels whereas $\qprop$ receives 1-dim rewards.

\begin{figure}[!htb]
	\centering
	\begin{center}
		\subfigure{\includegraphics[width=2.85in]{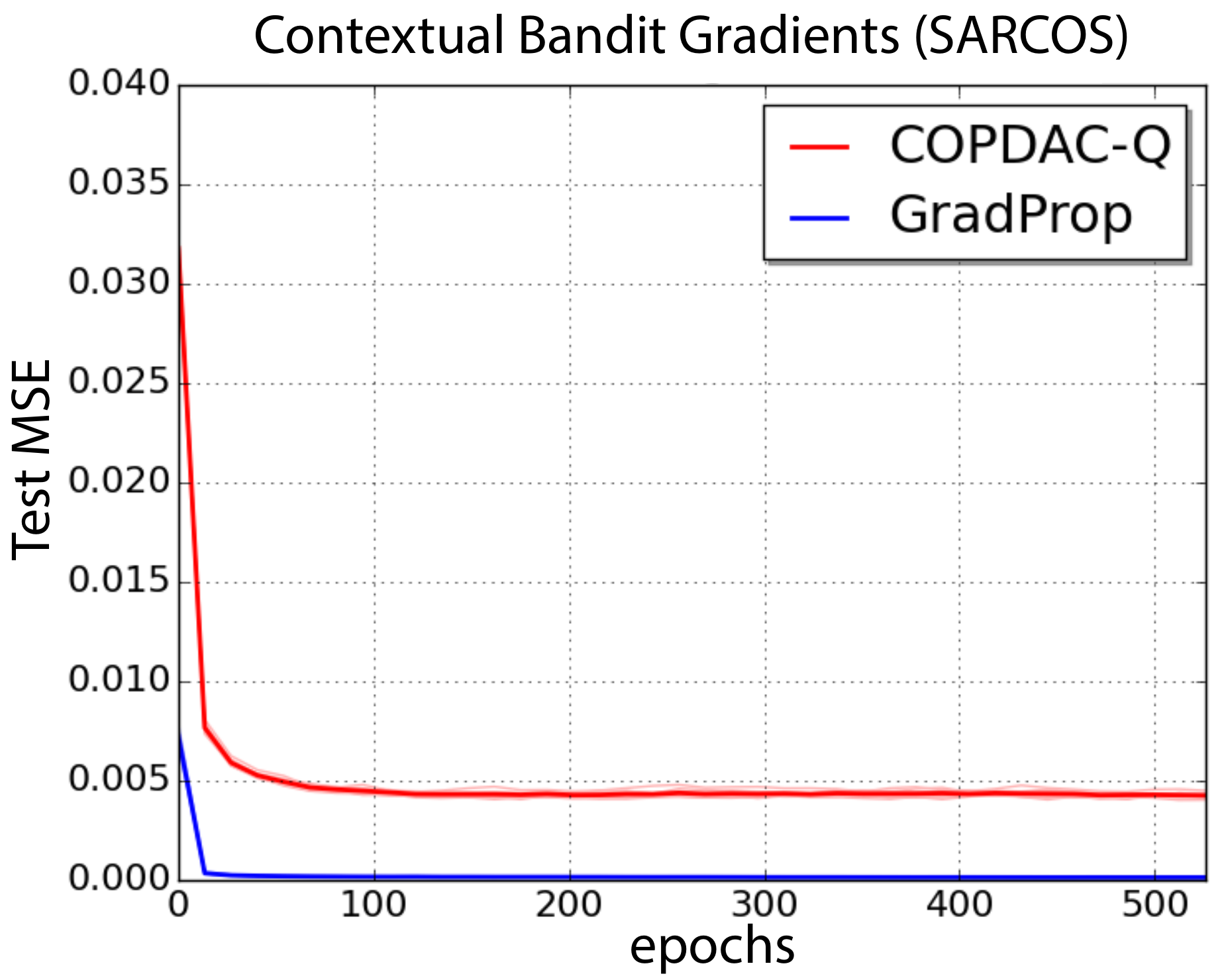}}
		\hspace{5mm}
		\subfigure{\includegraphics[width=2.85in]{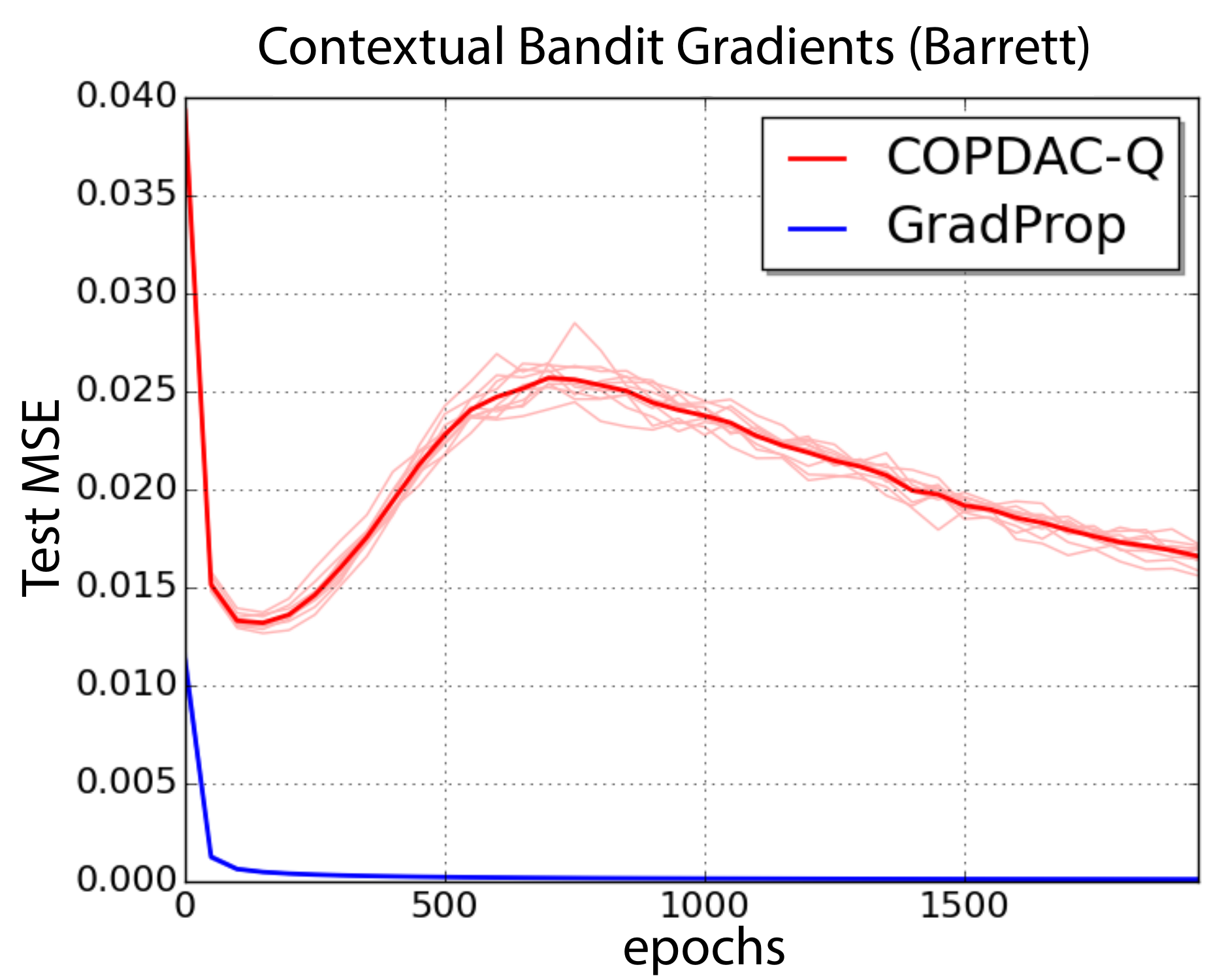}}
	\end{center}
	\vspace{-3mm}
	\caption{\textbf{Gradient estimates for contextual bandit tasks.}
	The normalized MSE of the gradient estimates compared against the true gradients, i.e. $\frac{1}{7}\|\grad_{est}-\grad_{true}\|_2^2$, are shown for 10 runs of $\copdac$ and $\qprop$, along with their averages (thick lines).
	}
	\label{f:robo_grad}
\end{figure}

\paragraph{Accuracy of gradient-estimates.}
The true value-gradients can be computed and compared with the algorithm's estimates on the contextual bandit task. Fig.~\ref{f:robo_grad} shows the performance of the two algorithms. $\qprop$'s gradient-error converges to $<0.005$ on both tasks. $\copdac$'s gradient estimate, implicit in the advantage function, converges to 0.03 (SARCOS) and 0.07 (BARRETT). This confirms that $\qprop$ yields significantly better gradient estimates.

$\copdac$'s estimates are significantly worse for Barrett compared to SARCOS, in line with the worse performance of $\copdac$ on Barrett in Fig.~\ref{f:robo}. It is unclear why $\copdac$'s gradient estimate gets worse on Barrett for some period of time. On the other hand, since there are no guarantees on $\copdac$'s estimates, it follows that its erratic behavior is perhaps not surprising.

\paragraph{Comparison with bandit task in \citep{silver:14}.}
Note that although the contextual bandit problems investigated here are lower-dimensional (with 21-dimensional state spaces and 7-dimensional action spaces) than the bandit problem in \citep{silver:14} (with no state space and 10, 25 and 50-dimensional action spaces), they are nevertheless much harder. The optimal action in the bandit problem, in all cases, is the constant vector $[4,\ldots,4]$ consisting of only 4s. In contrast, SARCOS and BARRETT are nontrivial benchmarks even when fully supervised.

\subsection{Octopus Arm}

The octopus arm task is a challenging environment that is high-dimensional, sequential and highly nonlinear. 

\paragraph{Desciption.}
The objective is to learn to hit a target with a simulated octopus arm \citep{engel:05}.\footnote{Simulator taken from\\ \url{http://reinforcementlearningproject.googlecode.com/svn/trunk/FoundationsOfAI/octopus-arm-simulator/octopus/}}
Settings are taken from \citep{silver:14}. Importantly, the action-space is \emph{not} simplified using ``macro-actions''.
The arm has $C=6$ compartments attached to a rotating base. There are $50=8C+2$ state variables ($x$, $y$ position/velocity of nodes along the upper/lower side of the arm; angular position/velocity of the base) and $20=3C+2$ action variables controlling the clockwise and counter-clockwise rotation of the base and three muscles per compartment. 

After each step, the agent receives a reward of $10\cdot \Delta_{dist}$, where $\Delta_{dist}$ is the change in distance between the arm and the target. 
The final reward is $+50$ if the agent hits the target.
An episode ends when the target is hit or after 300 steps. 

The arm initializes at eight positions relative to the target: $\pm 45^\circ, \pm 75^\circ, \pm 105^\circ, \pm 135^\circ$. See Appendix~\ref{sec:octopus_details} for more details.

\begin{figure}[!htb]
	\centering
	\begin{center}
		\subfigure{\includegraphics[width=2.85in]{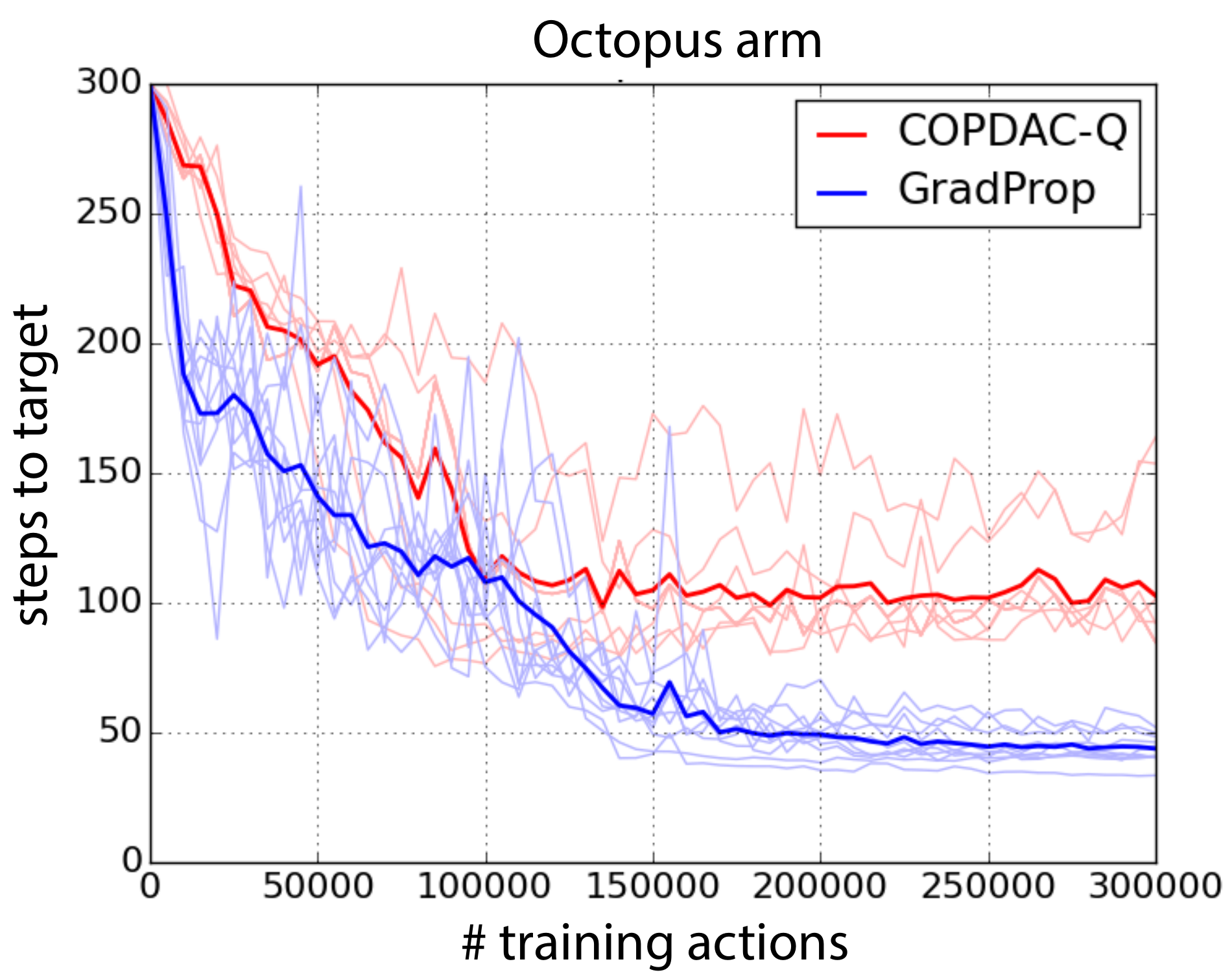}}
		\hspace{5mm}
		\subfigure{\includegraphics[width=2.85in]{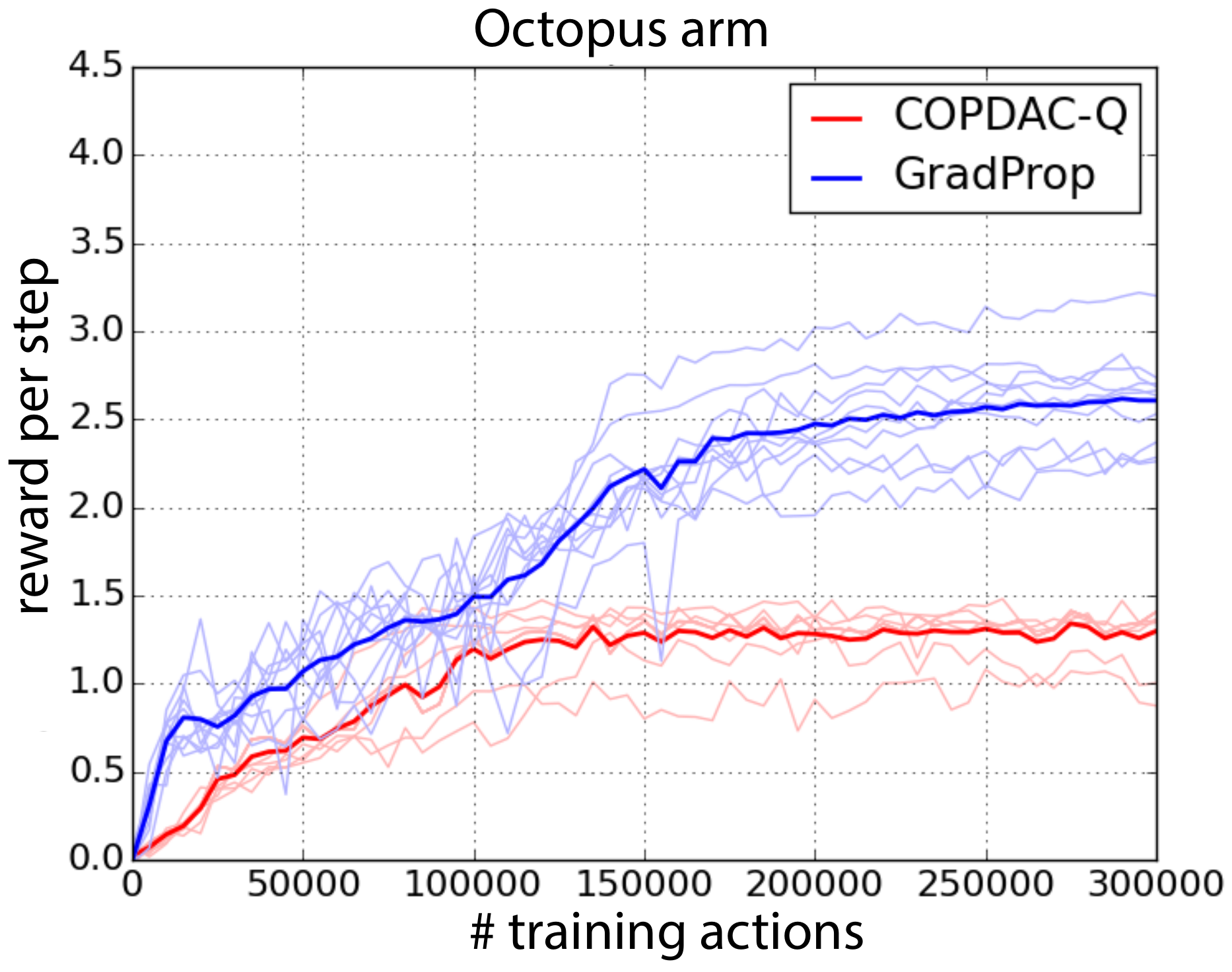}}
	\end{center}
	\vspace{-3mm}
	\caption{\textbf{Performance on octopus arm task.} Ten runs of $\qprop$ and  $\copdac$ on a 6-segment octopus arm with 20 action and 50 state dimensions. 
	Thick lines depict average values. Left panel: number of steps/episode for the arm to reach the target. Right panel: corresponding average rewards/step.}
	\label{f:oct}
\end{figure}

\paragraph{Network architectures.}
We applied $\qprop$ to an actor-network with $100$ hidden rectifiers and linear output units clipped to lie in $[0,1]$; and critic and deviator networks both with two hidden layers of $100$ and $40$ rectifiers, and linear output units.
Updates were computed via RMSProp with step rate of $10^{-4}$, moving average decay, with Nesterov momentum~\citep{hinton:12a} penalty of $0.9$ and $0.9$ respectively, and discount rate $\gamma$ of $0.95$.

The variance of the Gaussian noise was initialized to $\sigma^2=1.0$. An explore/exploit tradeoff was implemented as follows. When the arm hit the target in more than 300 steps, we set $\sigma^2\leftarrow \sigma^2\cdot 1.3$; otherwise $\sigma^2\leftarrow \sigma^2/1.3$. A hard lower bound was fixed at $\sigma^2=0.3$.

We implemented COPDAC-Q on a variety of architectures; the best results are shown (also please see Figure~3 in \citep{silver:14}). They were obtained using a similar architecture to $\qprop$, with sigmoidal hidden units and sigmoidal output units for the actor. Linear, rectilinear and clipped-linear output units were also tried.  As for $\qprop$, cloning and experience replay were used to increase stability.


\paragraph{Performance.}

Figure \ref{f:oct} shows the steps-to-target and average-reward-per-step on ten training runs.
$\qprop$ converges rapidly and reliably (within $\pm170,000$ steps) to a stable policy that uses less than 50 steps to hit the target on average (see supplementary video for examples of the final policy in action). $\qprop$ converges quicker, and to a better solution, than $\copdac$. The reader is strongly encouraged to compare our results with those reported in \citep{silver:14}. To the best of our knowledge, $\qprop$ achieves the best performance to date on the octopus arm task.

\paragraph{Stability.}
It is clear from the variability displayed in the figures that both the policy and the gradients learned by $\qprop$ are more stable than $\copdac$. Note that the higher variability exhibited by $\qprop$ in the right-hand panel of Fig.~\ref{f:oct} (rewards-per-step) is misleading. It arises because dividing by the number of steps -- which is lower for $\qprop$ since it hits the target more quickly after training -- inflates $\qprop$'s apparent variability.

\section{Conclusion}
\label{sec:conc}

Value-Gradient Backpropagation ($\qprop)$ is the first deep reinforcement learning algorithm with compatible function approximation for continuous policies. It builds on the deterministic actor-critic, $\copdac$, developed in \citep{silver:14} with two decisive modifications. First, we incorporate an explicit estimate of the value gradient into the algorithm. Second, we construct a model that decouples the internal structure of the actor, critic, and deviator -- so that all three can be trained via backpropagation. 

$\qprop$ achieves state-of-the-art performance on two contextual bandit problems where it simultaneously solves seven regression problems without observing labels. Note that $\qprop$ is competitive with recent \emph{fully supervised} methods that solve a \emph{single} regression problem at a time. Further, $\qprop$ outperforms the prior state-of-the-art on the octopus arm task, quickly converging onto policies that rapidly and fluidly hit the target.

\paragraph{Acknowledgements.}
We thank Nicolas Heess for sharing the settings of the octopus arm experiments in \citep{silver:14}.

\setcounter{equation}{0}\setcounter{section}{0}
\renewcommand{\thesection}{\Alph{section}}
\renewcommand{\theequation}{\Alph{section}.\arabic{equation}}
\setcounter{thm}{0}\setcounter{defn}{0}
\renewcommand{\thethm}{\Alph{section}.\arabic{thm}}
\renewcommand{\thedefn}{\Alph{section}.\arabic{defn}}
\renewcommand{\therem}{\Alph{section}.\arabic{rem}}

\vspace{4mm}
\noindent
{\Large \textbf{Appendices}}

\section{Explicit weight updates under $\qprop$}
\label{sec:explicit}

It is instructive to describe the weight updates under $\qprop$ more explicitly. 

Let $\theta^j$, $\wt^j$ and $\vt^j$ denote the weight vector of unit $j$, according to whether it belongs to the actor, deviator or critic network. Similarly, in each case $\bpi^j$ or $\pi^j$ denotes the influence of unit $j$ on the network's output layer, where the influence is vector-valued for actor and deviator networks and scalar-valued for the critic network. 

Weight updates in the deviator-actor-critic model, where all three networks consist of rectifier units performing stochastic gradient descent, are then per Algorithm~\ref{alg:pb}. Units that are not active on a round do not update their weights that round.
\vspace{2mm}

\begin{algorithm}[H]
    \caption{\texttt{$\qprop$: Explicit updates}.\label{alg:pb}}
    \DontPrintSemicolon
     \SetKwInOut{Input}{input}
     \For{rounds $t =1, 2, \ldots, T$}{
        Network gets state $\sta_t$, responds $\act_t=\bmu_\Theta(\sta_t)+\bepsilon$, gets reward $r_t$\;
        $\tdg_t \longleftarrow r_t  + \gamma Q^{\Vt_t}(\tilde{\sta}_{t+1})
        - Q^{\Vt_t}(\tilde{\sta}_t) - \langle \Qt^{\Wt_t}(\tilde{\sta}_t),\bepsilon\rangle\qquad\quad\texttt{// compute TDG-error}$\;
        \For{unit $j =1, 2, \ldots, n$}{
            \uIf{$j$ is an active actor unit}{
                $\theta^{j}_{t+1} \longleftarrow \theta^{j}_{t} + \eta^A_t \cdot\Big\langle \Qt^{\Wt_t}\big(\tilde{\sta}_t\big),\bpi^j_t \Big\rangle\cdot \bphi^j_t(\sta_t)$
                \quad\qquad\texttt{// backpropagate $\mathtt \Qt^\Wt$}\;
            }
            \uElseIf{$j$ is an active critic unit}{
                $\vt^{j}_{t+1} \longleftarrow \vt^{j}_t+ \eta^C_t \cdot\Big\langle\tdg_t, \pi^j_t\Big\rangle\cdot\bphi^j_t(\sta_t)$
                \qquad\quad\qquad\qquad\texttt{// backpropagate $\tdg$}\;
            }
            \uElseIf{$j$ is an active deviator unit}{
                $\wt^{j}_{t+1} \longleftarrow \wt^{j}_t+ \eta^C_t \cdot\Big\langle\tdg_t\cdot \bepsilon,\bpi^j_t\Big\rangle\cdot\bphi^j_t(\sta_t)$
                \quad\quad\qquad\texttt{// backpropagate $\tdg\cdot\bepsilon$}\;
            }
        }
    }
\end{algorithm}
\vspace{2mm}

\section{Details for octopus arm experiments}
\label{sec:octopus_details}
Listing 1 summarizes technical information with respect to the physical description and task setting used in the octopus arm simulator in XML format.

\begin{sflisting}[caption=Physical description and task setting for the octopus arm (setting.xml).]
<constants>
    <frictionTangential>0.4</frictionTangential>
    <frictionPerpendicular>1</frictionPerpendicular>
    <pressure>10</pressure>
    <gravity>0.01</gravity>
    <surfaceLevel>5</surfaceLevel>
    <buoyancy>0.08</buoyancy>
    <muscleActive>0.1</muscleActive>
    <musclePassive>0.04</musclePassive>
    <muscleNormalizedMinLength>0.1</muscleNormalizedMinLength>
    <muscleDamping>-1</muscleDamping>
    <repulsionConstant>.01</repulsionConstant> 
    <repulsionPower>1</repulsionPower>
    <repulsionThreshold>0.7</repulsionThreshold>
    <torqueCoefficient>0.025</torqueCoefficient>
</constants>
    
<targetTask timeLimit="300" stepReward="1">
  <target position="-3.25 -3.25" reward="50" />
</targetTask>
\end{sflisting}


\begin{thebibliography}{45}
\providecommand{\natexlab}[1]{#1}
\providecommand{\url}[1]{\texttt{#1}}
\expandafter\ifx\csname urlstyle\endcsname\relax
  \providecommand{\doi}[1]{doi: #1}\else
  \providecommand{\doi}{doi: \begingroup \urlstyle{rm}\Url}\fi

\bibitem[Agogino and Tumer(2004)]{agogino:04}
Adrian~K Agogino and Kagan Tumer.
\newblock Unifying {T}emporal and {S}tructural {C}redit {A}ssignment
  {P}roblems.
\newblock In \emph{AAMAS}, 2004.

\bibitem[Agogino and Tumer(2008)]{agogino:08}
Adrian~K Agogino and Kagan Tumer.
\newblock Analyzing and {V}isualizing {M}ultiagent {R}ewards in {D}ynamic and
  {S}tochastic {E}nvironments.
\newblock \emph{Journal of Autonomous Agents and Multi-Agent Systems},
  17\penalty0 (2):\penalty0 320--338, 2008.

\bibitem[Baird(1995)]{baird:95}
L~C Baird.
\newblock Residual algorithms: {R}einforcement learning with function
  approximation.
\newblock In \emph{ICML}, 1995.

\bibitem[Balduzzi(2015)]{doco:15}
David Balduzzi.
\newblock {Deep Online Convex Optimization by Putting Forecaster to Sleep}.
\newblock In \emph{arXiv:1509.01851}, 2015.

\bibitem[Balduzzi et~al.(2015)Balduzzi, Vanchinathan, and Buhmann]{bvb:15}
David Balduzzi, Hastagiri Vanchinathan, and Joachim Buhmann.
\newblock Kickback cuts {B}ackprop's red-tape: {B}iologically plausible credit
  assignment in neural networks.
\newblock In \emph{AAAI}, 2015.

\bibitem[Barto et~al.(1983)Barto, Sutton, and Anderson]{barto:83}
Andrew~G Barto, Richard~S Sutton, and Charles~W Anderson.
\newblock Neuronlike {A}dapative {E}lements {T}hat {C}an {S}olve {D}ifficult
  {L}earning {C}ontrol {P}roblems.
\newblock \emph{IEEE Trans. Systems, Man, Cyb}, 13\penalty0 (5):\penalty0
  834--846, 1983.

\bibitem[Bastien et~al.(2012)Bastien, Lamblin, Pascanu, Bergstra, Goodfellow,
  Bergeron, Bouchard, and Bengio]{bastien:12}
F~Bastien, P~Lamblin, R~Pascanu, J~Bergstra, I~Goodfellow, A~Bergeron,
  N~Bouchard, and Y~Bengio.
\newblock {Theano: new features and speed improvements}.
\newblock In \emph{{NIPS Workshop: Deep Learning and Unsupervised Feature
  Learning}}, 2012.

\bibitem[Bergstra et~al.(2010)Bergstra, Breuleux, Bastien, Lamblin, Pascanu,
  Desjardins, Turian, Warde-Farley, and Bengio]{bergstra:10}
J~Bergstra, O~Breuleux, F~Bastien, P~Lamblin, R~Pascanu, G~Desjardins,
  J~Turian, D~Warde-Farley, and Yoshua Bengio.
\newblock Theano: {A} {C}{P}{U} and {G}{P}{U} {M}ath {E}xpression {C}ompiler.
\newblock In \emph{Proc. {P}ython for {S}cientific {C}omp. {C}onf. (SciPy)},
  2010.

\bibitem[Dahl et~al.(2013)Dahl, Sainath, and Hinton]{dahl:13}
George~E Dahl, Tara~N Sainath, and Geoffrey Hinton.
\newblock Improving deep neural networks for {L}{V}{C}{S}{R} using rectified
  linear units and dropout.
\newblock In \emph{I{E}{E}{E} {I}nt {C}onf on {A}coustics, {S}peech and
  {S}ignal {P}rocessing (ICASSP)}, 2013.

\bibitem[Dann et~al.(2014)Dann, Neumann, and Peters]{dann:14}
Christoph Dann, Gerhard Neumann, and Jan Peters.
\newblock Policy {E}valuation with {T}emporal {D}ifferences: {A} {S}urvey and
  {C}omparison.
\newblock \emph{JMLR}, 15:\penalty0 809--883, 2014.

\bibitem[Deisenroth et~al.(2011)Deisenroth, Neumann, and Peters]{deisenroth:11}
Marc~Peter Deisenroth, Gerhard Neumann, and Jan Peters.
\newblock {A Survey on Policy Search for Robotics}.
\newblock \emph{Foundations and Trends in Machine Learning}, 2\penalty0
  (1-2):\penalty0 1--142, 2011.

\bibitem[Dud{\'\i}k et~al.(2014)Dud{\'\i}k, Erhan, Langford, and Li]{dudik:14}
Miroslav Dud{\'\i}k, Dumitru Erhan, John Langford, and Lihong Li.
\newblock {Doubly Robust Policy Evaluation and Optimization}.
\newblock \emph{Statistical Science}, 29\penalty0 (4):\penalty0 485--511, 2014.

\bibitem[Engel et~al.(2005)Engel, Szab{\'o}, and Volkinshtein]{engel:05}
Y~Engel, P~Szab{\'o}, and D~Volkinshtein.
\newblock Learning to control an octopus arm with gaussian process temporal
  difference methods.
\newblock In \emph{NIPS}, 2005.

\bibitem[Fairbank and Alonso(2012)]{fairbank:12}
Michael Fairbank and Eduardo Alonso.
\newblock {Value-Gradient Learning}.
\newblock In \emph{IEEE World Conference on Computational Intelligence (WCCI)},
  2012.

\bibitem[Fairbank et~al.(2013)Fairbank, Alonso, and Prokhorov]{fairbank:13}
Michael Fairbank, Eduardo Alonso, and Daniel~V Prokhorov.
\newblock {An Equivalence Between Adaptive Dynamic Programming With a Critic
  and Backpropagation Through Time}.
\newblock \emph{IEEE Trans. Neur. Net.}, 24\penalty0 (12):\penalty0 2088--2100,
  2013.

\bibitem[Flaxman et~al.(2005)Flaxman, Kalai, and McMahan]{flaxman:05}
Abraham Flaxman, Adam Kalai, and H~Brendan McMahan.
\newblock Online convex optimization in the bandit setting: {G}radient descent
  without a gradient.
\newblock In \emph{SODA}, 2005.

\bibitem[Glorot et~al.(2011)Glorot, Bordes, and Bengio]{glorot:11}
Xavier Glorot, Antoine Bordes, and Yoshua Bengio.
\newblock {D}eep {S}parse {R}ectifier {N}eural {N}etworks.
\newblock In \emph{Proc. 14th {I}nt {C}onference on {A}rtificial {I}ntelligence
  and Statistics (AISTATS)}, 2011.

\bibitem[Guestrin et~al.(2002)Guestrin, Lagoudakis, and Parr]{guestrin:02}
Carlos Guestrin, Michail Lagoudakis, and Ronald Parr.
\newblock Coordinated {R}einforcement {L}earning.
\newblock In \emph{ICML}, 2002.

\bibitem[Hafner and Riedmiller(2011)]{hafner:11}
Roland Hafner and Martin Riedmiller.
\newblock Reinforcement learning in feedback control: {C}hallenges and
  benchmarks from technical process control.
\newblock \emph{Machine Learning}, 84:\penalty0 137--169, 2011.

\bibitem[Hinton et~al.(2012)Hinton, Srivastava, and Swersky]{hinton:12a}
G~Hinton, Nitish Srivastava, and Kevin Swersky.
\newblock Lecture 6a: Overview of minibatch gradient descent.
\newblock 2012.

\bibitem[HolmesParker et~al.(2014)HolmesParker, Agogino, and
  Tumer]{holmesparker:14}
Chris HolmesParker, Adrian~K Agogino, and Kagan Tumer.
\newblock Combining {R}eward {S}haping and {H}ierarchies for {S}caling to
  {L}arge {M}ultiagent {S}ystems.
\newblock \emph{The Knowledge Engineering Review}, 2014.

\bibitem[Jordan and Jacobs(1990)]{jordan:90}
Michael~I Jordan and R~A Jacobs.
\newblock {Learning to control an unstable system with forward modeling}.
\newblock In \emph{NIPS}, 1990.

\bibitem[Kakade(2001)]{kakade:01}
Sham Kakade.
\newblock A natural policy gradient.
\newblock In \emph{NIPS}, 2001.

\bibitem[Konda and Tsitsiklis(2000)]{konda:00}
Vijay~R Konda and John~N Tsitsiklis.
\newblock Actor-critic algorithms.
\newblock In \emph{NIPS}, 2000.

\bibitem[Kpotufe and Boularias(2013)]{kpotufe:13}
Samory Kpotufe and Abdeslam Boularias.
\newblock Gradient {W}eights help {N}onparametric {R}egressors.
\newblock In \emph{Advances in Neural Information Processing Systems (NIPS)},
  2013.

\bibitem[Levine et~al.(2015)Levine, Finn, Darrell, and Abbeel]{levine:15}
Sergey Levine, Chelsea Finn, Trevor Darrell, and Pieter Abbeel.
\newblock {End-to-End Training of Deep Visuomotor Policies}.
\newblock \emph{ar{X}iv:1504.00702}, 2015.

\bibitem[Mnih et~al.(2015)Mnih, Kavukcuoglu, Silver, Rusu, Veness, Bellemare,
  Graves, Riedmiller, Fidjeland, Ostrovski, Petersen, Beattie, Sadik,
  Antonoglou, King, Kumaran, Wierstra, Legg, and Hassabis]{Mnih:2015wq}
Volodymyr Mnih, Koray Kavukcuoglu, David Silver, Andrei~A. Rusu, Joel Veness,
  Marc~G. Bellemare, Alex Graves, Martin Riedmiller, Andreas~K. Fidjeland,
  Georg Ostrovski, Stig Petersen, Charles Beattie, Amir Sadik, Ioannis
  Antonoglou, Helen King, Dharshan Kumaran, Daan Wierstra, Shane Legg, and
  Demis Hassabis.
\newblock Human-level control through deep reinforcement learning.
\newblock \emph{Nature}, 518\penalty0 (7540):\penalty0 529--533, 02 2015.

\bibitem[Nair and Hinton(2010)]{nair:10}
Vinod Nair and Geoffrey Hinton.
\newblock {R}ectified {L}inear {U}nits {I}mprove {R}estricted {B}oltzmann
  {M}achines.
\newblock In \emph{ICML}, 2010.

\bibitem[Nemirovski and Yudin(1983)]{nemirovski:83}
A~S Nemirovski and D~B Yudin.
\newblock \emph{Problem complexity and method efficiency in optimization}.
\newblock Wiley-Interscience, 1983.

\bibitem[Nguyen-Tuong et~al.(2008)Nguyen-Tuong, Peters, and Seeger]{nguyen:08}
Duy Nguyen-Tuong, Jan Peters, and Matthias Seeger.
\newblock {Local Gaussian Process Regression for Real Time Online Model
  Learning}.
\newblock In \emph{NIPS}, 2008.

\bibitem[Peters and Schaal(2006)]{peters:06}
Jan Peters and Stefan Schaal.
\newblock Policy {G}radient {M}ethods for {R}obotics.
\newblock In \emph{Proc. IEEE/RSJ Int. Conf. Intell. Robots Syst.}, 2006.

\bibitem[Prokhorov and Wunsch(1997)]{prokhorov:97}
Daniel~V Prokhorov and Donald~C Wunsch.
\newblock {Adaptive Critic Designs}.
\newblock \emph{IEEE Trans. Neur. Net.}, 8\penalty0 (5):\penalty0 997--1007,
  1997.

\bibitem[Raginsky and Rakhlin(2011)]{raginsky:11}
Maxim Raginsky and Alexander Rakhlin.
\newblock {Information-Based Complexity, Feedback and Dynamics in Convex
  Programming}.
\newblock \emph{IEEE Trans. Inf. Theory}, 57\penalty0 (10):\penalty0
  7036--7056, 2011.

\bibitem[Silver et~al.(2014)Silver, Lever, Heess, Degris, Wierstra, and
  Riedmiller]{silver:14}
David Silver, Guy Lever, Nicolas Heess, Thomas Degris, Daan Wierstra, and
  Martin Riedmiller.
\newblock Deterministic {P}olicy {G}radient {A}lgorithms.
\newblock In \emph{ICML}, 2014.

\bibitem[Srivastava et~al.(2014)Srivastava, Hinton, Krizhevsky, Sutskever, and
  Salakhutdinov]{srivastava:14}
Nitish Srivastava, Geoffrey Hinton, Alex Krizhevsky, Ilya Sutskever, and Ruslan
  Salakhutdinov.
\newblock Dropout: {A} {S}imple {W}ay to {P}revent {N}eural {N}etworks from
  {O}verfitting.
\newblock \emph{JMLR}, 15:\penalty0 1929--1958, 2014.

\bibitem[Sutton and Barto(1998)]{sutton:98}
R~S Sutton and A~G Barto.
\newblock \emph{Reinforcement {L}earning: {A}n {I}ntroduction}.
\newblock MIT Press, 1998.

\bibitem[Sutton et~al.(1999)Sutton, Mc{A}llester, Singh, and
  Mansour]{sutton:99}
Richard Sutton, David Mc{A}llester, Satinder Singh, and Yishay Mansour.
\newblock Policy gradient methods for reinforcement learning with function
  approximation.
\newblock In \emph{NIPS}, 1999.

\bibitem[Sutton et~al.(2009{\natexlab{a}})Sutton, Maei, Precup, Bhatnagar,
  Silver, Szepesv\'ari, and Wiewiora]{sutton:09}
Richard Sutton, Hamid~Reza Maei, Doina Precup, Shalabh Bhatnagar, David Silver,
  Csaba Szepesv\'ari, and Eric Wiewiora.
\newblock Fast {G}radient-{D}escent {M}ethods for {T}emporal-{D}ifference
  {L}earning with {L}inear {F}unction {A}pproximation.
\newblock In \emph{ICML}, 2009{\natexlab{a}}.

\bibitem[Sutton et~al.(2009{\natexlab{b}})Sutton, Szepesv\'ari, and
  Maei]{sutton:09a}
Richard Sutton, Csaba Szepesv\'ari, and Hamid~Reza Maei.
\newblock A convergent ${O}(n)$ algorithm for off-policy temporal-difference
  learning with linear function approximation.
\newblock In \emph{Adv in Neural Information Processing Systems (NIPS)},
  2009{\natexlab{b}}.

\bibitem[Trivedi et~al.(2014)Trivedi, Wang, Kpotufe, and
  Shakhnarovich]{trivedi:14}
Shubhendu Trivedi, Jialei Wang, Samory Kpotufe, and Gregory Shakhnarovich.
\newblock {A Consistent Estimator of the Expected Gradient Outerproduct}.
\newblock In \emph{UAI}, 2014.

\bibitem[Tsitsiklis and Roy(1997)]{tsitsiklis:97}
John Tsitsiklis and Benjamin~Van Roy.
\newblock An {A}nalysis of {T}emporal-{D}ifference {L}earning with {F}unction
  {A}pproximation.
\newblock \emph{IEEE Trans. Aut. Control}, 42\penalty0 (5):\penalty0 674--690,
  1997.

\bibitem[Wahlstr{\"o}m et~al.(2015)Wahlstr{\"o}m, Sch{\"o}n, and
  Deisenroth]{wahlstrom:15}
Niklas Wahlstr{\"o}m, Thomas~B. Sch{\"o}n, and Marc~Peter Deisenroth.
\newblock {From Pixels to Torques: Policy Learning with Deep Dynamical Models}.
\newblock \emph{ar{X}iv:1502.02251}, 2015.

\bibitem[Wang and Si(2001)]{wang:01}
Y~Wang and J~Si.
\newblock {On-line learning control by association and reinforcement}.
\newblock \emph{IEEE Trans. Neur. Net.}, 12\penalty0 (2):\penalty0 264--276,
  2001.

\bibitem[Williams(1992)]{williams:92}
Ronald~J Williams.
\newblock Simple {S}tatistical {G}radient-{F}ollowing {A}lgorithms for
  {C}onnectionist {R}einforcement {L}earning.
\newblock \emph{Machine Learning}, 8:\penalty0 229--256, 1992.

\bibitem[Zeiler et~al.(2013)Zeiler, Ranzato, Monga, Mao, Yang, Le, Nguyen,
  Senior, Vanhoucke, Dean, and Hinton]{zeiler:13}
M~D Zeiler, M~Ranzato, R~Monga, M~Mao, K~Yang, Q~V Le, P~Nguyen, A~Senior,
  V~Vanhoucke, J~Dean, and G~Hinton.
\newblock On {R}ectified {L}inear {U}nits for {S}peech {P}rocessing.
\newblock In \emph{ICASSP}, 2013.

\end{thebibliography}
\end{document}